%% file: main.tex
\documentclass{article} 
\usepackage{iclr2027_conference,times}

\usepackage[utf8]{inputenc} 
\usepackage[T1]{fontenc}    
\usepackage{url}            
\usepackage{booktabs}       
\usepackage{amsfonts}       
\usepackage{nicefrac}       
\usepackage{microtype}      
\usepackage{xcolor}         

\usepackage{tikz}
\usetikzlibrary{calc}

\usepackage{colortbl}

\usepackage{amsmath}
\usepackage{amssymb}
\usepackage{mathtools}
\usepackage{amsthm}

\usepackage{wrapfig}
\usepackage{listings}
\usepackage{tabulary}
\usepackage{etoolbox}
\usepackage{multirow}
\usepackage{enumitem}
\usepackage{xspace}
\usepackage{algorithm,algorithmicx,algpseudocode}
\usepackage{subcaption}

\usepackage[pagebackref=true,breaklinks=true,colorlinks,urlcolor={red!90!black},linkcolor={red!90!black},citecolor={blue!90!black},bookmarks=false]{hyperref}
\usepackage[toc,page,header]{appendix}
\usepackage{minitoc}

\newtheorem{theorem}{Theorem}[section]
\newtheorem{proposition}[theorem]{Proposition}
\newtheorem{lemma}[theorem]{Lemma}

\newtheorem{definition}[theorem]{Definition}

\newcommand{\RR}{\mathbb{R}}

\newcommand{\PR}[2]{\mathcal{P}_{#1}(\mathbb{R}^{#2})}
\newcommand{\QT}[2]{\PR{#1}{#2}/{\sim_T}}

\newcommand{\Qset}{\QT{2}{d}}

\newcommand{\GCone}{C(\mathcal{N})}
\newcommand{\supp}{\operatorname{supp}}
\newcommand{\Sym}{\mathrm{Sym}}
\newcommand{\Id}{\mathrm{Id}}

\title{Relative Wasserstein Angle and the problem of $W_2$-Nearest Gaussian Distribution}

\author{Binshuai Wang, Peng Wei\\
Department of Computer Science \\
The George Washington University \\
Washington D.C. \\
\texttt{\{derekwang, pwei\}@gwu.edu} \\
}
\input{math_commands}

\iclrfinalcopy
\begin{document}

\maketitle

\begin{abstract}
Understanding the distributional structure of high-dimensional datasets has become an important topic, yet direct visual characterization is difficult.
In this work, we develop a geometric framework for characterizing the distributional structure of empirical datasets by quantifying their deviation from the Gaussian family under the geometry induced by optimal transport theory.
Building on the cone structure of the relative translation invariant quadratic Wasserstein $(RW_2)$ space, we define two geometric quantities---the \emph{relative Wasserstein angle} and the \emph{orthogonal projection distance}---and show that they are well-defined because of the flat geometry of the filling cone between distributional rays.
This formulation recasts the problem of measuring deviation from the Gaussian family as an orthogonal projection problem onto the Gaussian cone and reveals that the commonly used moment-matching Gaussian is, in general, not the $W_2$-nearest Gaussian to a non-Gaussian distribution.
In one dimension, we derive closed-form expressions for the proposed quantities and extend closed-form expressions to several other location--scale families, including uniform, Laplace, and logistic distributions.
In higher dimensions, we develop a numerical approximation method for the proposed quantities based on empirical optimal transport and covariance-shape optimization.
Our experimental results show the empirical convergence and stability of the proposed methods and reveal that the $RW_2$ angle provides a robust and consistent measure of distributional non-Gaussianity.
Moreover, these results provide empirical support for its potential use as an indicator of distributional heterogeneity.
\end{abstract}

\input{sec/intro}

\input{sec/pre}
\input{sec/method}

\input{sec/exps}

\input{sec/con}


\bibliographystyle{iclr2027_conference}
\bibliography{myBib}

\appendix 


\input{sec/app_alg}

\input{sec/app_flatness}
\input{sec/app_proofs}

\input{sec/app_mm_vs_nearest}
\input{sec/app_others}

\input{sec/app_exp}

\end{document}

%% file: math_commands.tex
\usepackage{amsmath,amsfonts,bm}

\def\eqref#1{equation~\ref{#1}}

\def\1{\bm{1}}

\DeclareMathAlphabet{\mathsfit}{\encodingdefault}{\sfdefault}{m}{sl}
\SetMathAlphabet{\mathsfit}{bold}{\encodingdefault}{\sfdefault}{bx}{n}

\newcommand{\R}{\mathbb{R}}



%% file: sec/intro.tex
\section{Introduction}
With the rapid development of artificial intelligence and data science, datasets have become central to both academic research and industrial applications~\citep{Gebru21data, Paullada21data}. Understanding the structure of a data distribution, including its shape and density, can reveal important aspects of dataset quality~\citep{gong23quality}, such as heterogeneity and coverage, which can influence model performance and generalization. 
Accurately characterizing the distributional structure of datasets has therefore become an important topic in modern data analysis.

Despite its importance, understanding the distributional structure of a dataset through direct visualization is often challenging, particularly when the data lie in a high-dimensional space. 
In geometry, objects are commonly characterized through comparison with familiar reference shapes, such as spheres or rectangles. Motivated by this perspective, we seek to characterize the distributional structure of high-dimensional datasets by comparing their empirical distributions with standard distribution families. 
In particular, we focus on the Gaussian family as a typical single-source reference family for assessing distributional structure.


Gaussian distributions are among the most widely used probability distributions in statistics and machine learning. 
Many models either assume that the underlying data follow a Gaussian distribution or rely on Gaussian noise, priors, or latent variables in machine learning~\citep{bishop2006pattern, arjovsky2017wgan, murphy13machine, Tipping99PPCA, Kingma14Bayers, ho2020denoising, lipman2023flow}. 
Nevertheless, given an empirical distribution, the theoretical study of how close an empirical distribution is to the Gaussian distribution family remains limited. 
Classical statistical quantities, such as skewness, kurtosis, and other higher-order moments, can capture specific differences from Gaussianity~\citep{mardia1970measures, jarque1987test,hall1992bootstrap}; however, these quantities can only provide specific assessments and lack a unified geometric framework for quantifying the discrepancy between an empirical distribution and the Gaussian family.

In contrast to this classical statistical perspective, optimal transport (OT) theory provides a natural and unified geometric framework for assessing the distributional difference between two probability distributions~\citep{villani2003topics, villani2009OT, peyre2019computational}. 
The Wasserstein distance, arising from OT theory, treats probability distributions as points in a metric space and quantifies their discrepancy through the minimum cost of transporting mass from one distribution to another. 
Because of its intuitive geometric interpretation, optimal transport has become an ideal tool for analyzing distributional differences. This perspective naturally leads to the following question:

\emph{From the perspective of optimal transport, given an empirical distribution \(\mu\), how close is it to the Gaussian family? Equivalently, to what extent can \(\mu\) be regarded as Gaussian?}

In this paper, we show that this question can be reformulated as a problem of finding the minimal orthogonal projection distance from the distributional point \(\mu\) to the convex cone of Gaussian distributions under the quadratic Wasserstein geometry.
Building on this geometric perspective, we introduce two novel geometric notions, the relative quadratic Wasserstein angle (the \(RW_2\) angle) and the orthogonal projection distance, to quantify the degree of non-Gaussianity for an empirical distribution \(\mu\).
We further derive closed-form formulae for both quantities in the one-dimensional case and develop an approximation algorithm for higher-dimensional settings.
Finally, we conduct two experiments to validate the proposed theories and methods. 
Our experimental results demonstrate the convergence and stability of the proposed methods and show that the \(RW_2\) angle can provide a robust and consistent measure of distributional non-Gaussianity. 
Moreover, the results suggest that the \(RW_2\) angle can capture meaningful aspects of distributional heterogeneity.

\paragraph{Organization.}
The remainder of this paper is organized as follows.
Section~\ref{sec:pre} reviews classical results from optimal transport theory, including the \(RW_2\) metric space and the convex cone of Gaussian distributions.
Section~\ref{sec:method} introduces the \(RW_2\) angle, the orthogonal projection distance, the filling cone, and derives closed-form expressions and computational algorithms for these quantities.
Section~\ref{sec:exp} presents experimental results that validate our theoretical analyses and methods.

\paragraph{Notations.}
Let $\PR{p}{d}$ denote the set of probability distributions on $\RR^d$ with finite $p$th-order moments.
For a distribution $\mu$, let $\bar{\mu}$ and $\Sigma_\mu$ denote its mean and covariance matrix, respectively.
The moment-matching Gaussian of distribution $\mu$ is
$\mathcal{N}(\bar{\mu},\Sigma_\mu)$.
We write $[\mu]$ for the translation-equivalence class of $\mu$ and $RW_p([\mu],[\nu])$ for the relative translation invariant $p$-Wasserstein distance between $[\mu]$ and $[\nu]$.
The quotient space equipped with the $RW_2$ metric is denoted by $\Qset$, and $\|[\mu]\|_{RW_2}$ denotes the distance from $[\mu]$ to the equivalence class of the Dirac distribution $[\delta_0]$.
The ray induced by positive dilation of $[\mu]$ is denoted by
$
[[\mu]]
:=
\bigl\{(s\,\mathrm{Id})_\#[\mu] : s>0\bigr\}.
$
The relative Wasserstein angle between the rays $[[\mu]]$ and $[[\nu]]$ is denoted by
$\angle([[\mu]],[[\nu]])$, and $p([\mu],[[\nu]])$ denotes the orthogonal projection distance from $[\mu]$ to $[[\nu]]$.
Finally, $\mathcal{N}(0,\Sigma)$ denotes a zero-mean Gaussian distribution with covariance matrix $\Sigma$, and
$
\GCone
$
is used to denote the Gaussian cone.

\paragraph{Related Work.}
Estimating deviation from Gaussianity has a long history in statistics. 
Classical approaches rely on higher-order moments, such as skewness and kurtosis, to detect asymmetry and tail behavior beyond the Gaussian model \citep{mardia1970measures, jarque1987test, dagostino1990tests}. 
Beyond moment-based criteria, information-theoretic measures characterize non-Gaussianity through entropy and divergence. 
Representative approaches include negentropy \citep{hyvarinen2000independent}, which exploits the maximum-entropy property of Gaussian distributions, as well as divergence-based measures such as the Kullback--Leibler divergence \citep{kullback1951information}, Rényi divergences, and related quantities in information geometry \citep{amari2016information}. 
While powerful and widely used, these approaches are typically tied to specific functionals of the distribution, but do not provide a unified geometric framework for different types of distributions.


Unlike the statistical approaches discussed above, optimal transport and Wasserstein distances measure discrepancies between distributions from a geometric viewpoint~\citep{villani2003topics, villani2009OT, santambrogio2015optimal}.
Several related works have studied Wasserstein distances between general probability distributions and Gaussian distributions, with particular focus on covariance structures~\citep{bobkov2019onedim, gelbrich1990formula}.
A closely related notion is the Bures--Wasserstein angle introduced in~\citet{Oostrum22BW}, which can be viewed as a special case of the $RW_2$ angle when restricted to Gaussian distributions.
Despite these developments, the moment-matching Gaussian remains a widely used and natural Gaussian approximation. 
In this work, we show that the moment-matching Gaussian is not, in general, the $W_2$-nearest Gaussian to a non-Gaussian distribution. 
To the best of our knowledge, our work is the first to characterize and quantify the non-Gaussianity of data distributions through geometric measures based on angles and projections in optimal transport.

%% file: sec/pre.tex
\section{Preliminaries}
\label{sec:pre}

\subsection{Optimal Transport Theory}
\label{sec:pre_rw_metric}
Optimal transport theory provides a geometric framework for comparing probability distributions by minimizing the transportation cost from one measure to another.
Let \(\mu,\nu\in\mathcal{P}_p(\mathbb{R}^d)\) be two probability measures with finite \(p\)-th moments.
Given the cost function \(c(x,y)=\|x-y\|^p\), the \(p\)-norm optimal transport problem is defined as follows.

\begin{definition}[$p$-norm Optimal Transport Problem~\citep{villani2003topics}]
\label{def_OT}
\begin{equation*}
\mathrm{OT}(\mu,\nu,p)
:=
\min_{\gamma\in\Gamma(\mu,\nu)}
\int \|x-y\|^p \, d\gamma(x,y),
\end{equation*}
where \(\Gamma(\mu,\nu)\) denotes the set of all transport plans with marginals \(\mu\) and \(\nu\), namely,
\[
\Gamma(\mu,\nu)
=
\Big\{
\gamma\in\PR{p}{2d}
\;\big|\;
\int \gamma(x,y)\,dy=\mu(x),\;
\int \gamma(x,y)\,dx=\nu(y)
\Big\}.
\]
\end{definition}
The corresponding \(p\)-Wasserstein distance is defined by
\[
W_p(\mu,\nu)
:=
\mathrm{OT}(\mu,\nu,p)^{1/p},
\qquad p\in[1,\infty),
\]
which defines a metric on \(\mathcal{P}_p(\mathbb{R}^d)\)~\citep{villani2009OT}.

Building on this framework, \citet{wang2026rwp} introduced a translation-invariant extension by identifying probability measures that differ only by a translation.
Let \(\sim_T\) denote the translation equivalence relation and let \([\mu]\) be the corresponding equivalence class of \(\mu\).
The \(p\)-th order relative translation invariant Wasserstein distance is defined by
\[
RW_p([\mu],[\nu])
:=
\min_{t\in\mathbb{R}^d}
W_p(\mu+t,\nu),
\]
which is a well-defined metric on the quotient space \(\mathcal{P}_p(\mathbb{R}^d)/{\sim_T}\).

For the quadratic case \(p=2\), the Wasserstein distance admits the decomposition
\begin{equation}
\label{eq:W2_decompose}
W_2^2(\mu,\nu)
=
\|\bar{\mu}-\bar{\nu}\|_2^2
+
RW_2^2([\mu],[\nu]),
\end{equation}
where \(\bar{\mu}\) and \(\bar{\nu}\) denote the means of \(\mu\) and \(\nu\), respectively.
Consequently, \(RW_2\) distance measures the intrinsic discrepancy between distributions after removing the translational difference.
More importantly, it induces a new metric space of probability distributions, $(\mathcal{P}_2(\mathbb{R}^d)/{\sim_T}, RW_2\bigr)$, 
which we refer to as the $RW_2$ metric space and study throughout the remainder of this paper.

\subsection{The Gaussian Convex Cone in the \(RW_2\) Space}
\label{sec:pre_gaussian_cone}
As shown by \citet{Takatsu2011Gaussian}, the metric completion of zero-mean Gaussian distributions under the $W_2$ distance forms a convex cone:
\[
C(\mathcal{N})
:=
\bigl\{
\mathcal{N}(0,\Sigma)
\,\big|\,
\Sigma\in\Sym_*^d
\bigr\},
\]
where \(\Sym_*^d\) denotes the set of symmetric nonnegative semi-definite \(d\times d\) covariance matrices.

The apex of this cone is the Dirac measure \(\delta_0\), corresponding to the degenerate Gaussian distribution with zero covariance.
For any \(\Sigma\in\Sym_*^d\) and \(s>0\), scaling the covariance matrix gives
\[
\mathcal{N}(0,s^2\Sigma)
=
(s\,\Id)_{\#}\mathcal{N}(0,\Sigma),
\]
where \((s\,\Id)_{\#}\) is the push-forward induced by the dilation map \(x\mapsto sx\).
Therefore, for each nonzero covariance matrix \(\Sigma\), the family
$
\bigl\{
\mathcal{N}(0,s^2\Sigma):s>0
\bigr\}
$
forms a geodesic ray emanating from \(\delta_0\) in the \(RW_2\) space.
The matrix \(\Sigma\) determines the shape of the Gaussian distribution, while \(s\) controls its scale along the ray.
As a result, the family of Gaussian distributions naturally admits a convex cone in the \(RW_2\) space.

%% file: sec/method.tex
\section{Methodology}
\label{sec:method}

In this section, we introduce our main theory.
We first introduce the cone geometry of the \(RW_2\) metric space and define the relative Wasserstein angle, or \(RW_2\) angle.
By showing that the filling cone between two distributional rays has zero sectional curvature, we establish that orthogonal projections and an inner-product-like quantity are also well-defined.
We then show that both the \(RW_2\) angle and the projection distance can be used to quantify the degree of non-Gaussianity for empirical distributions.
Finally, we present algorithms for computing the \(RW_2\) angle and the projection distance in both one-dimensional and high-dimensional settings.

\subsection{Relative Wasserstein Angles and Orthogonal Projections}
\label{sec:rw2angle}

\input{imgs/main/fig_rw2_cone_geometry}

\paragraph{Cone Geometry.}
As mentioned in Section~\ref{sec:pre_rw_metric}, \(RW_2\) defines a metric on \(\QT{2}{d}\), thereby yielding a metric space, denoted by \((\QT{2}{d}, RW_2)\).
Compared with the space \((\PR{2}{d}, W_2)\), a distinctive feature of this quotient space is its cone geometry.
Specifically, all Dirac distributions in \((\PR{2}{d}, W_2)\) collapse into a single equivalence class \([\delta_0]\), which serves as the \emph{apex} (the origin) of the conic space \((\QT{2}{d}, RW_2)\).
Other distributions form \emph{rays} emanating from this apex, each corresponding to a distributional ``shape'' up to translation and positive scaling.

To formalize this observation, we introduce the \emph{dilation (positive scaling) equivalence relation} \(\sim_S\) on the space \(\QT{2}{d}\), defined by
\[
[\mu] \sim_S [\nu]
\quad \Longleftrightarrow \quad
[\nu] = (s\,\Id)_\# [\mu]
\quad \text{for some } s>0.
\]
This relation partitions \(\QT{2}{d}\setminus\{[\delta_0]\}\) into \textit{disjoint rays}. For each non-apex distributional point \([\mu]\), the associated ray \([[\mu]]\) is defined as the collection of all positive dilations of \([\mu]\),
\[ 
[[\mu]] := \bigl\{\, (s\,\Id)_\# [\mu] : s>0 \,\bigr\}. 
\]
Here, the double-bracket notation \( [[\cdot]] \) is used to reflect two quotients by translation and dilation.

Geometrically, each ray \( [[\mu]] \) is a geodesic in the \(RW_2\) metric space, as the shortest path between any two points on the ray is obtained by continuously varying the dilation $s$.
Consequently, the space $(\Qset, RW_2)$ admits a cone structure, consisting of an apex and rays generated by dilations.

\paragraph{Relative Wasserstein Angle.} 
Inspired by the notion of angles in metric spaces \citep{ambrosio2005gradientflows}, we introduce an angle concept on this quotient space.

\begin{definition}[Relative Wasserstein Angle]
\label{def:rw_angles}
For any two rays \( [[\mu]] , [[\nu]] \in \Qset \), the relative Wasserstein angle between them, denoted by the $RW_2$ angle, is defined as
\begin{equation*}
\angle([[\mu]],[[\nu]])
:=\arccos\frac{a^2 + b^2 - c^2}{2ab},
\end{equation*}
where \( a = RW_2([\mu], [\delta_0]) \), \( b = RW_2([\nu], [\delta_0]) \), and \( c = RW_2([\mu], [\nu]) \).
\end{definition}
This definition follows the law of cosines and yields a well-defined angle notion between rays in $(\Qset, RW_2)$. 
The resulting angle quantifies the \emph{directional} difference between two rays and is invariant under positive rescaling along each ray. 


\paragraph{Orthogonal Projection.} 
The classical Wasserstein space is known to be an Alexandrov space of \textit{nonnegative} curvature, and thus Euclidean geometric properties generally do not hold globally~\citep{ambrosio2005gradientflows,Takatsu2011Gaussian}.
Nevertheless, we show that the relative Wasserstein angle behaves like a Euclidean angle, since it can be embedded in a two-dimensional Euclidean plane.

Given two rays in the \(RW_2\) space, we define the \emph{filling cone} between them as the union of all geodesics connecting points on the two rays, as illustrated in Figure~\ref{fig:rw2_projection}. 
The following result characterizes the geometry of this filling cone.

\begin{theorem}
If at least one of the two distinct non-apex rays is generated by an absolutely continuous probability distribution\footnote{The absolute continuity condition guarantees the uniqueness of the optimal transport plan and, consequently, the uniqueness of the corresponding filling cone.}, then the corresponding filling cone has zero sectional curvature. Moreover, it is isometric to a cone region in the two-dimensional Euclidean plane.
\end{theorem}
The proof is provided in Appendix~\ref{sec:prf_flat_add}.
This isometric property of the filling cone allows Euclidean notions, such as angles and projections, to be well-defined in the probability distribution space.

For projections, we use the closed ray 
$
\overline{[[\nu]]}
:=
[[\nu]]\cup\{[\delta_0]\}
=
\bigl\{(s\,\Id)_\#[\nu]:s\geq0\bigr\}, 
$
which includes the apex to ensure that the projection is attained even when the angle between the rays equals $\pi/2$.
\begin{definition}[Projection Distance]
\label{def:rw_projection}
Given a non-apex point $[\mu]$ and a ray $[[\nu]]$, 
the \emph{projection distance} from point \( [\mu] \) to ray \( [[\nu]] \) is defined as
\[
p([\mu],[[\nu]])
:=
\min_{[\eta]\in\overline{[[\nu]]}}
RW_2([\mu],[\eta])
=
\|[\mu]\|_{RW_2}
\sin\angle([[\mu]],[[\nu]]).
\]
For the apex, we set $p([\delta_0],[[\nu]])=0$.
\end{definition}
Here, \(\|[\mu]\|_{RW_2} := RW_2([\mu],[\delta_0])\) denotes the \textit{length} of \([\mu]\), the distance to the apex \([\delta_0]\).

Similarly, we may also define an inner-product-like quantity between distribution equivalence classes by
\[
    \langle [\mu],[\nu]\rangle
    :=
    \|[\mu]\|_{RW_2}\,\|[\nu]\|_{RW_2}
    \cos\angle([[\mu]],[[\nu]]).
\]
This notion can be interpreted as an average inner product induced by the optimal couplings, though this quantity does not make the quotient space an inner product space, since its underlying geometry is conic rather than linear.
Appendix~\ref{sec:three_rays_add} gives an example in which
three rays have unique flat pairwise filling cones, but the
union of these cones, equipped with the $RW_2$ metric, does not admit a Euclidean isometric embedding. 
Thus, pairwise flatness does not extend trivially to three or more rays.

\subsection{$W_2$-nearest Gaussian Distribution}
\label{sec:w2_nearest}
Having introduced these geometric notions, we now turn to the central question: given an empirical distribution $\mu$, to what extent can it be regarded as Gaussian?

From the perspective of optimal transport, this question can be viewed as finding the \emph{minimal} \(W_2\) distance between the distribution \(\mu\) and the Gaussian family. 
By the mean--variance decomposition in Equation~\ref{eq:W2_decompose}, any Gaussian minimizer of the \(W_2\) distance, which we refer to as a \(W_2\)-nearest Gaussian, must have the same mean as \(\mu\). 
Consequently, the problem can be further reduced to optimizing over the covariance matrices of Gaussian distributions.

\paragraph{Projection onto the Gaussian Cone.}
As discussed in Section~\ref{sec:pre_gaussian_cone}, the Gaussian family admits a convex cone $\GCone$ in the $RW_2$ metric space, including both the apex and Gaussian rays determined by covariance shapes.
As the minimal $W_2$ distance must be a projection distance from the point $[\mu]$ to a certain ray within the Gaussian cone $\GCone$, the central problem of finding the minimal $W_2$ distance from distribution $\mu$ to a Gaussian distribution can be reformulated as finding the projection distance from the point $[\mu]$ onto the Gaussian cone $\GCone$ in the metric space $(\Qset, RW_2)$.
Moreover, because the projection distance onto a ray is determined by the relative Wasserstein angle, this problem can also be formulated as finding the minimum $RW_2$ angle between the ray $[[\mu]]$ and a Gaussian ray in $\GCone$.

\paragraph{Solutions in One and Higher Dimensions.}
In the one-dimensional case, the Gaussian family corresponds to a single ray in the \(RW_2\) space. 
This geometric structure yields closed-form expressions for both the \(RW_2\) angle and the projection distance, as detailed in Section~\ref{subsec:one_dim}.

In higher dimensions, the problem generally does not admit a closed-form solution, except for certain specially structured distributions, such as coordinate-separable distributions in a principal component analysis (PCA) basis; see Appendix~\ref{sec:sol_Coordinate_add}. 
For the general case, we develop a numerical approximation method to solve the problem, as described in Section~\ref{subsec:high_dim}.

\paragraph{The Moment-Matching Gaussian.}
A commonly used Gaussian approximation of $\mu$ in statistics is the moment-matching Gaussian distribution
$\mu_{\mathrm{MM}}=\mathcal{N}(\bar{\mu},\Sigma_\mu)$,
where $\bar{\mu}$ and $\Sigma_\mu$ denote the mean and covariance of $\mu$, respectively.
It is natural to conjecture that this distribution is the $W_2$-nearest Gaussian to \( \mu \).
However, using the proposed notions of angle and projection, we show that the moment-matching Gaussian $\mu_{\mathrm{MM}}$ is \textit{not} the $W_2$-nearest Gaussian to a non-Gaussian $\mu$.

Here, we provide a proof sketch, while the detailed derivation is given in Appendix~\ref{sec:MM_Gaussian}. 
Since \(\mu_{\mathrm{MM}}\) shares the same covariance as \(\mu\), the two distributions have the same \(RW_2\) length:
$$
\|[\mu_{\mathrm{MM}}]\|_{RW_2}
=
\|[\mu]\|_{RW_2}
=
\sqrt{\operatorname{Tr}(\Sigma_\mu)}.
$$
However, when \(\mu\) is non-Gaussian, the \(RW_2\) angle between the ray \([[\mu]]\) and the Gaussian ray \([[\mu_{\mathrm{MM}}]]\) is strictly positive. 
Consequently, the orthogonal projection of \([\mu]\) onto this Gaussian ray has a smaller length than \([\mu_{\mathrm{MM}}]\). 
Therefore, \(\mu_{\mathrm{MM}}\) cannot be the \(W_2\)-nearest Gaussian to \(\mu\).

Moreover, we also find that, for certain distributions $\mu$, the $W_2$-nearest Gaussian may share neither the same eigenvalues nor the same eigenvectors of the covariance matrix with the moment-matching Gaussian; an example is provided in Appendix~\ref{sec:not_sharing_add}.
These examples suggest that analytic solutions are unlikely to exist in general for higher-dimensional cases.
Therefore, we adopt a numerical approximation method, which is discussed in Section~\ref{subsec:high_dim}.

\subsection{Closed-Form Solutions in One Dimension}
\label{subsec:one_dim}

In the one-dimensional case, all nondegenerate zero-mean Gaussian distributions lie on a single ray in the \(RW_2\) metric space, namely
\([[\mathcal{N}(0,\sigma^2)]]\),
with \(\sigma>0\). Exploiting the monotone optimal transport map in one dimension, we derive closed-form expressions for both the \(RW_2\) angle and the corresponding projection distance, as stated in Proposition~\ref{prop:1d_closed_form}. The detailed derivation is provided in Appendix~\ref{sec:prf_one_dim_add}, and Algorithm~\ref{alg:rw2_1d} in Appendix~\ref{sec:alg_add} summarizes the computational procedure.

\begin{proposition}[Closed-form solutions in one dimension]
\label{prop:1d_closed_form}
Let $\mu$ be a zero-mean empirical distribution on $\mathbb{R}$ with length $\sigma_\mu > 0$, given by
\(
\mu
:=
\frac{1}{n}\sum_{i=1}^n \delta_{x_i},
\bar{\mu}
=
\frac{1}{n}\sum_{i=1}^n x_i
=
0,
\sigma_\mu
=
\sqrt{\frac{1}{n}\sum_{i=1}^n x_i^2},
\)
where $x_1\le \cdots \le x_n$ are empirical samples.
Let $\phi$ and $\Phi$ denote the probability density function and cumulative distribution function of the standard normal distribution $\mathcal{N}(0,1)$, respectively.
Define
    \[
        z_0\leftarrow-\infty,\qquad
        z_i\leftarrow\Phi^{-1}\!\left(\frac{i}{n}\right),
        \quad i=1,\dots,n-1,\qquad
        z_n\leftarrow+\infty.
    \]
Define the projection length
$
l:=
\sum_{i=1}^n
x_i\bigl(\phi(z_{i-1})-\phi(z_i)\bigr).
$
Then for the ray $[[\mu]]$ and the Gaussian ray
$[[\mathcal{N}(0,\sigma^2)]]$, their relative Wasserstein angle $\theta$ and the corresponding projection distance $p$ from $[\mu]$ to this Gaussian ray, are given by
\begin{align}
\label{eq:one_dim_sol}
    \theta
    =
    \arccos\!\left(\frac{l}{\sigma_\mu}\right),
    \qquad
    p
    =
    \sigma_\mu\sin\theta.
\end{align}
\end{proposition}
\paragraph{Remark.}
The above formulae connect directly to classical work on Wasserstein-based tests of normality~\citep{Brown96Test,Barrio99test}.
In particular, \citet{Barrio99test} derived the normalized squared distance to the Gaussian family. 
In our notation, their statistic is
$
R_n
=
\frac{p^2}{\sigma_\mu^2}
=
\sin^2\theta.
$
Thus, the one-dimensional $RW_2$ angle is a monotone transformation of this existing statistic.
Our contribution here is to interpret these quantities as an angle and an orthogonal projection distance in
the cone geometry of the $RW_2$ space, within a framework
that also applies in higher dimensions.
\paragraph{Complexity Analysis.}
The computation requires \(\mathcal{O}(n\log n)\) time, dominated by sorting the empirical samples, and \(\mathcal{O}(n)\) memory.

The same approach can be extended to other distribution families parametrized by mean and scale, commonly referred to as \emph{location--scale} families, such as the uniform, Laplace, and logistic distributions.
Their corresponding formulae are provided in Appendix~\ref{sec:sol_other_dis_add}.
For the distribution families that are not location--scale cases, such as geometric distributions, the problem becomes one of projecting onto more general curves in the \(RW_2\) space.
We leave this direction for future work.

\subsection{Solutions in Higher Dimensions}
\label{subsec:high_dim}
Unlike the one-dimensional case, in which the Gaussian family forms a single ray in the \(RW_2\) metric space, the Gaussian family in higher dimensions consists of multiple rays determined by their covariance matrices. Consequently, finding the \(W_2\)-nearest Gaussian to a given distribution \(\mu\) requires not only computing the projection distance from \([\mu]\) onto a fixed Gaussian ray, but also optimizing over the space of Gaussian covariance matrices.

Concretely, for a given distribution $\mu$, the central problem is given by
\begin{align}
\label{eq:opt_rays}
\min_{\Sigma\in\Sym_{*}^d}\ \min_{s>0}
RW_2^2\bigl([\mu],[\mathcal{N}(0,s^2\Sigma)]\bigr),
\end{align}
where $\Sym_{*}^d$ denotes the set of symmetric nonnegative semi-definite covariance matrices, and $s$ controls the scale along the Gaussian ray associated with $\Sigma$.

We can parameterize the covariance matrix through its eigendecomposition,
$\Sigma = U\Lambda U^\top$,
where $U\in\mathcal{O}(d)$ is an orthogonal matrix and
$\Lambda=\operatorname{diag}(\lambda_1,\dots,\lambda_d)$ contains the nonnegative eigenvalues $\lambda_i\geq 0$.
Since the Gaussian ray depends only on the covariance shape, we can reduce the scale parameter $s$ by imposing the normalization
$
    \operatorname{tr}(\Sigma)=\sum_{i=1}^d\lambda_i=1.
$

Under these parameter reductions, Problem~\ref{eq:opt_rays} simplifies to unit-trace optimization
\begin{align*}
\min_{U\in\mathcal{O}(d),\ \Lambda\succeq 0,\ \operatorname{Tr}(\Lambda)=1}
RW_2^2\bigl([\mu],[\mathcal{N}(0,U\Lambda U^\top)]\bigr).
\end{align*}
For numerical computation, we approximate the continuous Gaussian distribution $\mathcal{N}(0,U\Lambda U^\top)$ using samples
$
\{y_i\}_{i=1}^m \sim \mathcal{N}(0,U\Lambda U^\top).
$
We initialize the samples by drawing from the normalized Gaussian $\mathcal{N}(0,\frac{1}{d}I_d)$.
At each iteration, the Gaussian samples are updated alternately between an optimal-transport gradient step and a retraction onto the normalized Gaussian covariance constraint. 
Here, the gradient update is computed from the optimal transport coupling between the current Gaussian samples and empirical samples from \(\mu\).
After each gradient update, we estimate the covariance
matrix of the centered samples, normalize its eigenvalues
to have unit sum, and draw new Gaussian samples from the
resulting covariance matrix.
This procedure is repeated until the stopping criterion is met or the iteration limit is reached. 
Once an optimal Gaussian covariance shape $\Sigma^\star$ with $\operatorname{tr}(\Sigma^\star)=1$ and its corresponding
$RW_2$ angle $\theta^\star$ are obtained, a $W_2$-nearest Gaussian can be recovered by 
$
\nu^\star
=
\mathcal N\!\left(
\bar\mu,\,
r^2\cos^2\theta^\star\,\Sigma^\star
\right),
r=\|[\mu]\|_{RW_2}
=\sqrt{\operatorname{tr}(\Sigma_\mu)}.
$
The complete algorithm and its complexity analysis are provided in Appendix~\ref{sec:alg_add_2}.
Figure~\ref{fig:mnist_rw2_convergence} further shows the convergence behavior and stability of the proposed algorithm.

\paragraph{Remarks.}
In low-dimensional settings, semi-discrete optimal transport methods based on power diagrams, also known as Laguerre diagrams, can provide more accurate estimates of the \(RW_2\) distance by avoiding the sampling approximation of the continuous Gaussian distribution~\citep{peyre2019computational}. However, these methods become computationally prohibitive in high-dimensional settings, such as \(d=512\), due to the cost of constructing and integrating over Laguerre cells. In addition, very small covariance eigenvalues can lead to ill-conditioned covariance matrices, making the numerical optimization sensitive to gradient updates. 
In practice, it is also useful to incorporate PCA-based dimensionality reduction to reduce computational cost and improve numerical stability.

%% file: imgs/main/fig_rw2_cone_geometry.tex
\usetikzlibrary{arrows.meta,calc,angles,quotes}

\definecolor{rwBlue}{HTML}{0072B2}
\definecolor{rwVermilion}{HTML}{D55E00}
\definecolor{rwTeal}{HTML}{009E73}
\definecolor{rwPurple}{HTML}{AA4499}
\definecolor{rwOchre}{HTML}{997000}
\definecolor{rwInk}{HTML}{333333}
\definecolor{rwRed}{HTML}{C44E52} 

\begin{figure}[t]
\centering

\begin{subfigure}[t]{0.47\textwidth}
\centering
\vspace{0pt}
\begin{tikzpicture}[
  x=0.86cm,
  y=0.86cm,
  >=Latex,
  font=\scriptsize,
  ray/.style={
    line width=1.1pt,
    -{Latex[length=2mm,width=1.4mm]}
  }
]

\path[use as bounding box] (-0.45,-1.10) rectangle (6.15,3.65);

\coordinate (O) at (0,0);


\draw[gray!45, densely dashed]
  (O) -- (-12:5.65);
\draw[gray!45, densely dashed]
  (O) -- (65:3.75);

\draw[ray,rwTeal]
  (O) -- (58:3.65)
  node[above right,text=rwTeal] {Laplace ray};

\draw[ray,rwPurple]
  (O) -- (41:4.15)
  node[above right,text=rwPurple] {Uniform ray};

\draw[ray,rwVermilion]
  (O) -- (25:4.70)
  node[right,text=rwVermilion] {Gaussian ray};

\draw[ray,rwBlue]
  (O) -- (9:4.90)
  node[right,text=rwBlue,align=left] {Empirical ray\\(fixed shape)};

\draw[ray,rwOchre]
  (O) -- (-6:4.95)
  node[right,text=rwOchre] {Logistic ray};

\fill[rwInk] (O) circle (1.8pt);
\node[below=5pt,text=rwInk] at (O) {Apex $[\delta_0]$};

\end{tikzpicture}
\caption{Distribution families as rays.}
\label{fig:rw2_family_rays}
\end{subfigure}
\hfill
\begin{subfigure}[t]{0.51\textwidth}
\centering
\vspace{0pt}
\begin{tikzpicture}[
  x=0.86cm,
  y=0.86cm,
  >=Latex,
  font=\scriptsize,
  ray/.style={
    line width=1.1pt,
    -{Latex[length=2mm,width=1.4mm]}
  },
  point/.style={
    circle,
    inner sep=1.7pt
  }
]

\path[use as bounding box] (-0.45,-1.10) rectangle (6.15,3.65);

\coordinate (O) at (0,0);
\coordinate (M) at (34:4.70);
\coordinate (P) at (M |- O);
\coordinate (U) at (34:5.65);
\coordinate (V) at (5.65,0);

\fill[rwBlue!8] (O) -- (V) -- (U) -- cycle;
\node[text=rwBlue!65!black, font=\small] at (2.75,0.95)
  {Filling cone};

\draw[ray,rwRed] (O) -- (U)
  node[above=4pt,text=rwRed] {Ray $[[\mu]]$};

\draw[ray,rwBlue] (O) -- (V)
  node[right=2pt,text=rwBlue] {Ray $[[\nu]]$};

\draw[rwVermilion,line width=1.8pt]
  (O) -- (M)
  node[
    midway,
    above=5pt,
    sloped,
    text=rwVermilion
  ] {Length $\|[\mu]\|_{RW_2}$};

\draw[rwTeal,line width=1.8pt] (O) -- (P);
\node[below=5pt,text=rwTeal!75!black,align=center]
  at ($(O)!0.55!(P)$)
  {Projection length $l$\\
   $l=\|[\mu]\|_{RW_2}\cos\theta$};

\draw[
  rwPurple,
  densely dashed,
  line width=1.1pt
] (M) -- (P);

\node[
  right=4pt,
  text=rwPurple,
  align=left
] at ($(M)!0.48!(P)$)
  {Projection\\distance $p$};

\draw[rwInk!75,line width=0.65pt]
  ($(P)+(-0.22,0)$)
  -- ($(P)+(-0.22,0.22)$)
  -- ($(P)+(0,0.22)$);

\pic[
  draw=rwPurple,
  line width=0.9pt,
  angle radius=9mm,
  angle eccentricity=1.35,
  text=rwPurple,
  "$\theta$ ($RW_2$ angle)"{xshift=18pt}
] {angle=P--O--M};

\node[point,fill=rwInk] at (O) {};
\node[point,fill=rwVermilion] at (M) {};
\node[point,fill=rwTeal] at (P) {};

\node[below=5pt,text=rwInk] at (O)
  {Apex $[\delta_0]$};

\node[above left=3pt, text=rwVermilion] at (M)
  {Point $[\mu]$};

\node[below right=5pt,text=rwTeal!75!black] at (P)
  {Projection point};

\end{tikzpicture}
\caption{Angle and projection in a filling cone.}
\label{fig:rw2_projection}
\end{subfigure}

\caption{
Schematic illustration of the cone geometry of the $RW_2$ space in the one-dimensional case.
(a) Positive rescalings of a fixed distribution form a ray emanating from the apex $[\delta_0]$.
The labeled rays illustrate representative distribution families.
(b) The shaded region represents the \emph{filling cone}
between the rays $[[\mu]]$ and $[[\nu]]$, and their relative
Wasserstein angle is $\theta$.
The point $[\mu]$ lies at distance $\|[\mu]\|_{RW_2}$
from the apex.
Its orthogonal projection onto $[[\nu]]$ has radial distance
$l=\|[\mu]\|_{RW_2}\cos\theta$, while the dashed segment
indicates the projection distance $p$.
}
\label{fig:rw2_cone_geometry}
\end{figure}

%% file: sec/exps.tex
\section{Experiments}
\label{sec:exp}

In this section, we conduct two experiments to validate the proposed theoretical analysis and demonstrate the practical utility of our framework and algorithms. The first experiment focuses on the one-dimensional case, where the theoretical analyses can be examined directly. 
The second experiment investigates a practical application for characterizing the distributional structure of high-dimensional datasets, including aspects of heterogeneity. 
All experiments are conducted on a Linux workstation equipped with an Intel Core i7 CPU, 16 GB of RAM, and an NVIDIA RTX 3090 GPU.

\subsection{Non-Gaussianity Evaluation in One Dimension}
\label{sec:exp1}

\paragraph{Setup.}
We demonstrate the ability of the proposed geometric quantities to characterize non-Gaussianity in one dimension.
We consider the following equally weighted Gaussian mixture,
\[
    \mu=\frac{1}{2}\mathcal{N}(-m_c,1)
       +\frac{1}{2}\mathcal{N}(m_c,1),
\]
where $m_c\in\{1,4\}$ controls the separation between the two component means.
For each sample size $N\in\{2^2,2^4,2^6,2^8,2^{10}\}$, we draw $N$ samples from $\mu$ to construct an empirical distribution $\widehat{\mu}_N$.

We consider two methods for computing the $RW_2$ angle $\theta$ and projection distance $p$.
The first method uses the closed-form expressions in Algorithm~\ref{alg:rw2_1d}, treating the Gaussian distribution continuously.
The second one approximates the Gaussian distribution using $N$ samples.
We repeat each experiment independently 10 times and report the mean and standard deviation of both quantities.
To compare the robustness of the two quantities across repetitions, we also compute the coefficient of variation,
$
    \mathrm{CV}=\frac{\mathrm{std}}{|\mathrm{mean}|},
$
where smaller values indicate better robustness.

\begin{figure}[t]
    \centering
    \begin{subfigure}[t]{0.49\textwidth}
        \centering
        \includegraphics[width=\textwidth]{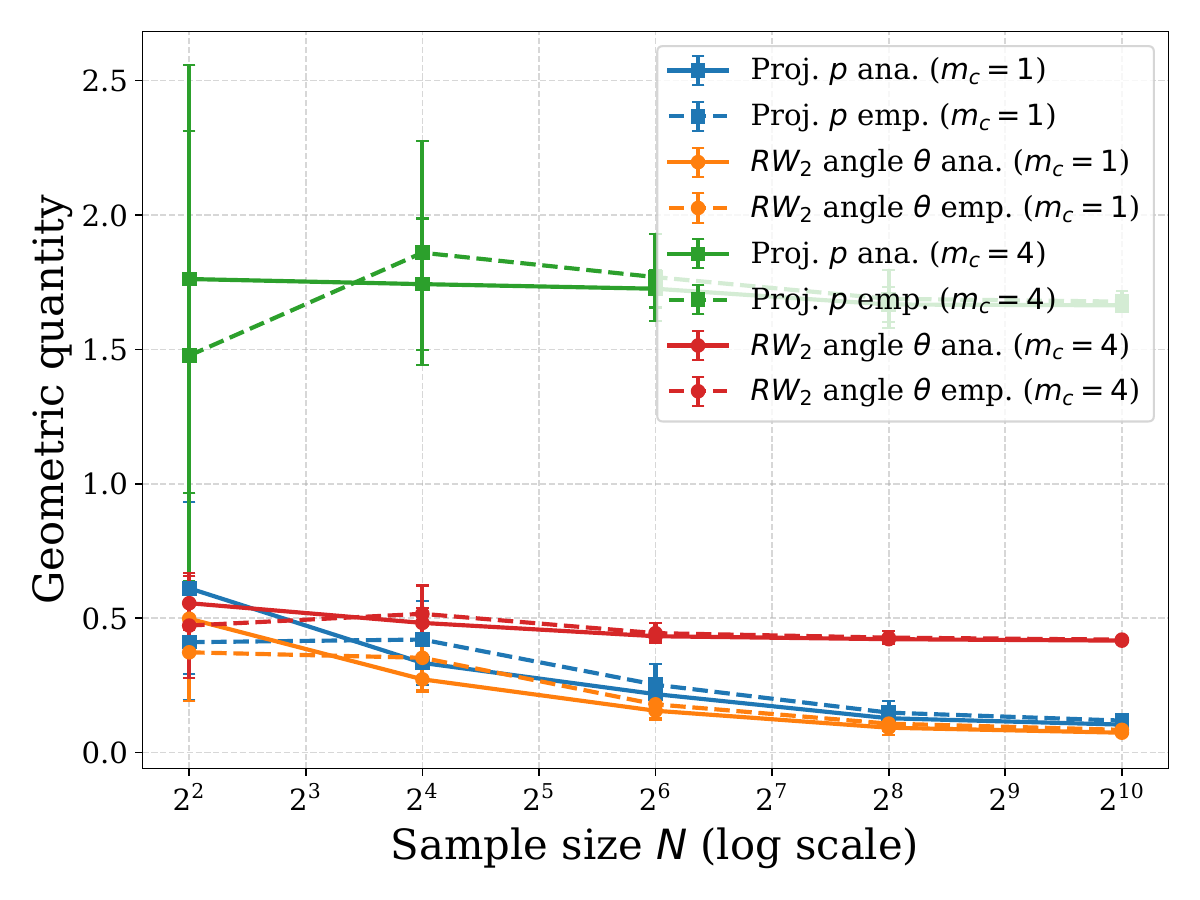}
        \caption{\(RW_2\) angle and projection distance as measurements of non-Gaussianity.}
        \label{fig:finite_sample_value}
    \end{subfigure}
    \hfill
    \begin{subfigure}[t]{0.49\textwidth}
        \centering
        \includegraphics[width=\textwidth]{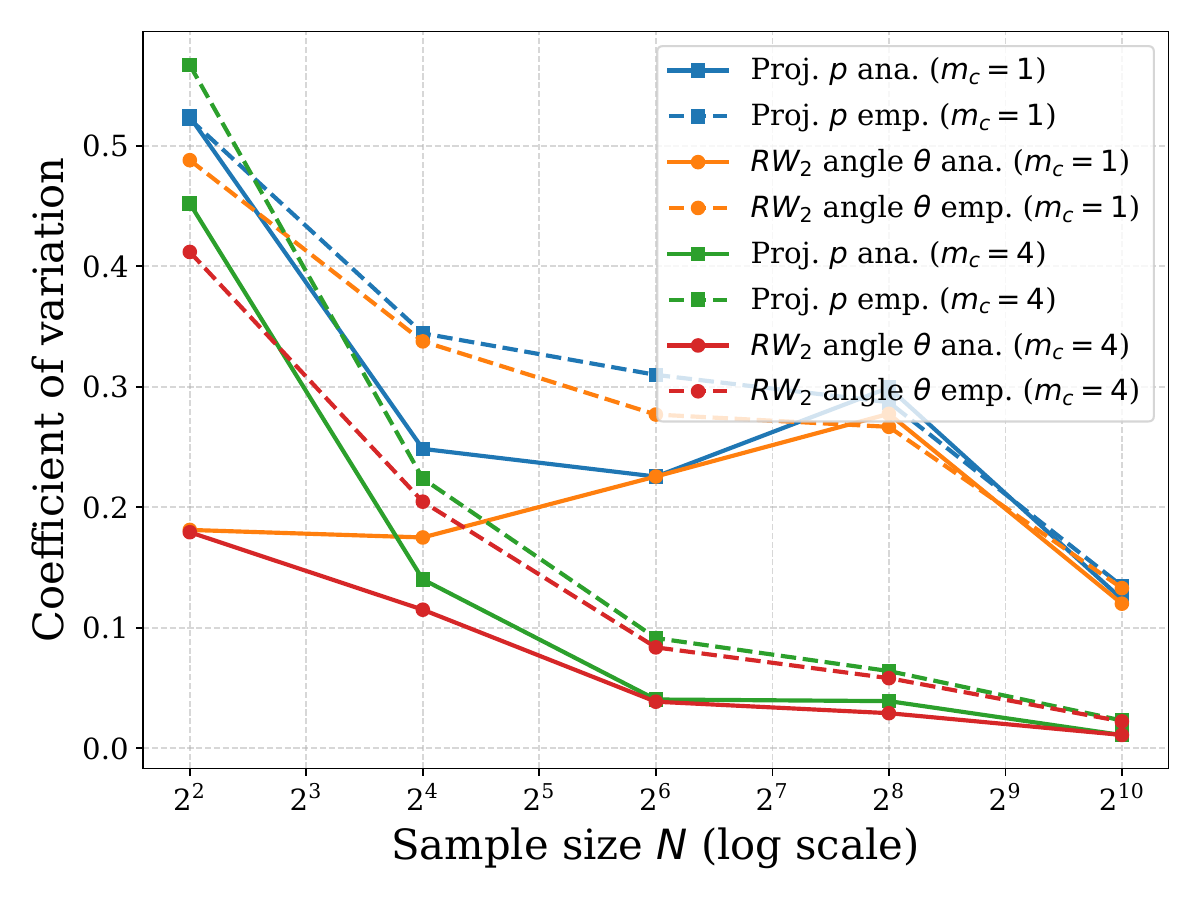}
        \caption{Comparison of coefficients of variation (CV).}
        \label{fig:finite_sample_robustness}
    \end{subfigure}
    \caption{
    Comparison of the \(RW_2\) angle \(\theta\) and projection distance \(p\) for quantifying the non-Gaussianity of a Gaussian mixture distribution.
    (a) Both quantities decrease as the mixture separation decreases, reflecting reduced non-Gaussianity. As the sample size \(N\) increases, the sampling-based estimates converge to the corresponding closed-form values.
    (b) Coefficients of variation, \(\mathrm{CV}=\mathrm{std}/|\mathrm{mean}|\), computed over 10 independent repetitions. The \(RW_2\) angle consistently exhibits a lower CV than the projection distance, indicating it is more robust than the projection distance.
    }
    \label{fig:finite_sample_results}
\end{figure}

\paragraph{Results.}
Figure~\ref{fig:finite_sample_value} shows that both the \(RW_2\) angle and the projection distance decrease as \(m_c\) decreases, indicating that both quantities can capture the reduction in non-Gaussianity. 
As the sample size \(N\) increases, the sampling-based estimates converge to the corresponding closed-form values.
These results show that the proposed quantities can be used to evaluate the non-Gaussianity.
Figure~\ref{fig:finite_sample_robustness} reports the coefficient of variation (CV) of the two quantities. 
We find that the \(RW_2\) angle consistently exhibits a lower CV than the projection distance. This suggests that the angle is more robust than the projection distance, potentially since the former measures differences between rays rather than individual points.
In addition, we also include a visual comparison between the moment-matching Gaussian and the \(W_2\)-nearest Gaussian in Appendix~\ref{app:exp_rw2_1d}.

\subsection{Non-Gaussianity Evaluation on High-Dimensional Datasets}

Beyond the one-dimensional setting, we evaluate the proposed approximation algorithm (Algorithm~\ref{alg:rw2_hd}) and investigate a practical application for characterizing the distributional structure of high-dimensional datasets.
Due to space limitations, we summarize the essential experimental setup and key results here and defer detailed configurations, results, and analyses to Appendix~\ref{sec:exp_add}.

\paragraph{Setup.}
We evaluate both raw and latent data representations.
Specifically, we use MNIST~\citep{lecun1994gradient} and EMNIST~\citep{cohen2017EMNIST} for evaluation in pixel space, and CIFAR-10~\citep{krizhevsky2009learning} and CINIC-10~\citep{darlow2018cinic10} for evaluation in feature space, with features extracted by a pretrained Inception-v3 encoder~\citep{szegedy2015rethinking}.

For each dataset, we consider the full empirical distribution, individual classes, and selected mixtures of classes.
We use Algorithm~\ref{alg:rw2_hd} to estimate the \(RW_2\) angle between each distribution and its \(W_2\)-nearest Gaussian approximation, with the moment-matching Gaussian approximation as a baseline.
We report the mean and standard deviation of the estimated angles and running times.
Each experiment is repeated five times with different random seeds.
We also track the \(RW_2\) angle over optimization iterations to assess the empirical convergence and stability of the proposed algorithm.

\paragraph{Results.}
Figure~\ref{fig:mnist_rw2_convergence} shows the evolution of the estimated \(RW_2\) angle between each MNIST digit-class distribution and its Gaussian approximation during the first experimental repeat.
Each curve corresponds to one digit class (0--9).
The angles decrease rapidly during the initial iterations and subsequently stabilize around class-dependent values, providing empirical evidence of the convergence and stability of Algorithm~\ref{alg:rw2_hd}.

Across MNIST, EMNIST, CIFAR-10, and CINIC-10, the estimated \(W_2\)-nearest Gaussian consistently yields smaller \(RW_2\) angles than the moment-matching baseline, at a higher computational cost.
The full empirical distributions generally exhibit larger angles than most individual classes, and the angles increase as more classes are combined in the selected mixture sequences.
Although paired datasets exhibit similar angles for their full distributions, differences emerge at the class level, such as digit~\(4\) in MNIST versus EMNIST and the horse class in CIFAR-10 versus CINIC-10.
Together, these results demonstrate that the \(RW_2\) angle captures differences in non-Gaussianity across datasets, classes, and mixtures in both pixel-based and latent representations, providing empirical support for its potential use as an indicator of distributional heterogeneity.

%% file: sec/con.tex
\section{Conclusion}
\label{sec:conclusion}
In this paper, we introduced a geometric framework for quantifying deviations from the Gaussian family under Wasserstein geometry.
By exploiting the cone structure of the \(RW_2\) space and the flatness of the filling cone, we established that the relative Wasserstein angle and orthogonal projection distance are well-defined and provide interpretable measures of non-Gaussianity.
We derived closed-form expressions in one dimension and extended the analysis to several other location--scale families.
For high-dimensional distributions, we developed approximation algorithms based on empirical optimal transport, with PCA used to reduce computational cost.
We also demonstrated that the moment-matching Gaussian is generally \textit{not} the \(W_2\)-nearest Gaussian to a non-Gaussian distribution.

Experiments on synthetic distributions and high-dimensional datasets validate the proposed framework.
Both the \(RW_2\) angle and projection distance capture deviations from Gaussianity, with the angle exhibiting greater robustness to sampling variability.
The \(RW_2\) angle also reveals differences in non-Gaussianity across datasets, classes, and mixtures in both pixel-based and latent representations, providing empirical support for its potential use as an indicator of distributional heterogeneity.

More broadly, these quantities offer a geometric perspective on Wasserstein discrepancy beyond distance alone.
Promising directions for future work include dataset characterization, distribution shift analysis, and applications to flow matching and diffusion models, where the geometry of probability distributions plays a central role.



\newpage

\section*{Ethics Statement}

This work develops mathematical and computational methods for characterizing probability distributions and evaluates them using established benchmark datasets, including MNIST, EMNIST, CIFAR-10, and CINIC-10. The study does not involve new human-subject experiments, interventions, or the collection of private or sensitive personal information. We use existing datasets and pretrained models in accordance with their intended research use and appropriately cite their original sources.

A potential concern is the interpretation of the proposed \(RW_2\) angle as a general measure of dataset quality. 
The \(RW_2\) angle specifically quantifies deviation from a reference distribution family, the Gaussian family in this work, under a chosen representation and experimental protocol. 
It should therefore not be interpreted by itself as a measure of fairness, social representativeness, completeness, or overall dataset quality. 
In particular, measurements in latent feature space depend on the selected feature extractor, dimensionality-reduction procedure, sampling strategy, and other preprocessing choices. 
We consequently restrict our conclusions to the geometric properties measured under the stated experimental settings and do not advocate using the proposed metric alone for consequential dataset selection or auditing decisions. The authors have read and agree to adhere to the ICLR Code of Ethics.

\section*{Reproducibility Statement}

We have taken several steps to support reproducibility of both the theoretical and experimental results. The definitions, assumptions, and principal theoretical results are presented in Sections~2 and~3, with detailed derivations and proofs provided in the corresponding appendices.
Experimental protocols are described in Section~4, including the datasets, representations, preprocessing procedures, PCA dimensionality reduction, sampling strategy, optimization procedure, evaluation settings, and reported repetitions. 
An anonymized implementation and the corresponding experimental configuration files are provided as supplementary material to facilitate reproduction of the numerical results.

\section*{AI Use Statement}

Generative AI tools, including ChatGPT, Claude and DeepSeek, were used during the preparation of this work to assist with refining mathematical statements and their exposition, providing feedback on experimental design, assisting with the implementation and debugging of experimental code, and presenting the experimental results. 
Generative AI tools were also used to improve the clarity and readability of the manuscript, assist with manuscript organization, summarize and identify relevant literature, and help with reference formatting and related writing tasks. We did not use generative AI to generate the experimental datasets used in this work; translation and qualitative or thematic data analysis are not applicable to this study.

All AI-assisted content was reviewed by the authors. In particular, the mathematical definitions, assumptions, derivations, theoretical claims, and proofs were independently checked by the authors; AI-assisted code was reviewed, executed, and tested; numerical results were obtained from the authors' experimental implementations; and references and literature-related information were checked against the original sources where appropriate. Generative AI was not treated as an authoritative source for mathematical or scientific claims.

%% file: sec/app_alg.tex
\newpage
\section{Algorithms}
\label{sec:alg_add}


        

\subsection{\(RW_2\) Angle and Projection Distance in One Dimension}
\label{sec:alg_add_1}

\begin{algorithm}[ht]
\caption{\(RW_2\) Angle and Projection Distance in One Dimension}
\label{alg:rw2_1d}
\begin{algorithmic}[1]
    \Require Samples $\{x_i\}_{i=1}^n\subset\mathbb{R}$
    \Ensure \(RW_2\) angle $\theta$ and projection distance $p$

    \State Compute the sample mean:
    $\bar{x}\leftarrow \frac{1}{n}\sum_{i=1}^n x_i$
    \State Center the samples:
    $x_i\leftarrow x_i-\bar{x}$, \quad $i=1,\dots,n$
    \State Sort the centered samples so that $x_1\leq\cdots\leq x_n$
    \State Compute the empirical standard deviation:
    $\sigma_\mu\leftarrow\sqrt{\frac{1}{n}\sum_{i=1}^n x_i^2}$
    \State Set the Gaussian quantile boundaries:
    \[
        z_0\leftarrow-\infty,\qquad
        z_i\leftarrow\Phi^{-1}\!\left(\frac{i}{n}\right),
        \quad i=1,\dots,n-1,\qquad
        z_n\leftarrow+\infty
    \]
    \State Compute the projection length:
    \[
        l\leftarrow\sum_{i=1}^n
        x_i\bigl(\phi(z_{i-1})-\phi(z_i)\bigr)
    \]
    \State Compute the angle:
    $\theta\leftarrow\arccos(l/\sigma_\mu)$
    \State Compute the projection distance:
    $p\leftarrow\sigma_\mu\sin\theta$
    \State \Return $\theta,p$
\end{algorithmic}
\end{algorithm}

Here, $\Phi$ and $\phi$ denote the cumulative distribution function
and probability density function of the standard normal distribution,
respectively, with $\phi(-\infty)=\phi(+\infty)=0$.
If the empirical variance is zero, the distribution corresponds to the
apex of the \(RW_2\) space; its projection distance is zero, while its
angle is undefined.

\newpage
\subsection{Numerical Approximation of the \(RW_2\) Angle in High Dimensions}
\label{sec:alg_add_2}

\begin{algorithm}[ht]
\caption{Numerical Approximation of the \(RW_2\) Angle in High Dimensions}
\label{alg:rw2_hd}
\begin{algorithmic}[1]
    \Require Samples \(\{x_i\}_{i=1}^n\subset\mathbb{R}^d\), Gaussian sample size \(m\), step size \(\eta\), tolerance \(\varepsilon\), maximum number of iterations \(T_{\max}\)
    \Ensure Estimated covariance shape \(\Sigma^\star\) and \(RW_2\) angle \(\theta^\star\)

    \State Center the input samples:
    \(x_i\leftarrow x_i-\frac{1}{n}\sum_{k=1}^n x_k\)
    \State Initialize \(\Sigma^{(0)}\leftarrow I_d/d\)
    \State Draw \(y_1^{(0)},\dots,y_m^{(0)}\sim\mathcal{N}(0,\Sigma^{(0)})\)
    \State Center the Gaussian samples:
    \(y_j^{(0)}\leftarrow y_j^{(0)}-\frac{1}{m}\sum_{k=1}^m y_k^{(0)}\)
    \State set \(\theta^{(0)}\leftarrow+\infty\)

    \For{\(t=1,\dots,T_{\max}\)}
        \State Compute the optimal transport coupling with uniform marginals:
        \[
            \gamma^{(t)}\leftarrow
            \textsc{OT}\bigl(\{x_i\}_{i=1}^n,
            \{y_j^{(t-1)}\}_{j=1}^m\bigr)
        \]
        \State Compute the sample gradients:
        \[
            g_j^{(t)}\leftarrow
            \sum_{i=1}^n\gamma_{ij}^{(t)}
            \bigl(y_j^{(t-1)}-x_i\bigr),
            \qquad j=1,\dots,m
        \]
        \State Update the samples:
        \(\widetilde y_j^{(t)}
        \leftarrow y_j^{(t-1)}-\eta g_j^{(t)}\)
        \State Center the updated samples:
        \(\widetilde y_j^{(t)}\leftarrow
        \widetilde y_j^{(t)}-\frac{1}{m}\sum_{k=1}^m
        \widetilde y_k^{(t)}\)
        \State Estimate the covariance matrix:
        \[
            \widehat{\Sigma}^{(t)}
            \leftarrow\frac{1}{m}\sum_{j=1}^m
            \widetilde y_j^{(t)}
            (\widetilde y_j^{(t)})^\top
        \]
        \State Compute the eigendecomposition:
        \[
            \widehat{\Sigma}^{(t)}
            =U^{(t)}
            \operatorname{diag}(\widehat\lambda_1^{(t)},\dots,
            \widehat\lambda_d^{(t)})
            (U^{(t)})^\top
        \]
        \State Normalize the eigenvalues to enforce unit trace:
        \[
            \lambda_i^{(t)}
            \leftarrow
            \frac{\widehat\lambda_i^{(t)}}
            {\sum_{k=1}^d\widehat\lambda_k^{(t)}},
            \qquad i=1,\dots,d
        \]
        \State Set \(\Sigma^{(t)}\leftarrow
        U^{(t)}\operatorname{diag}(\lambda_1^{(t)},\dots,
        \lambda_d^{(t)})(U^{(t)})^\top\)
        \State Draw \(y_1^{(t)},\dots,y_m^{(t)}
        \sim\mathcal{N}(0,\Sigma^{(t)})\)
        \State Center the Gaussian samples:
        \(y_j^{(t)}\leftarrow y_j^{(t)}
        -\frac{1}{m}\sum_{k=1}^m y_k^{(t)}\)
        \State Compute \(\theta^{(t)}\leftarrow
        \theta_{RW_2}(\{x_i\}_{i=1}^n,\{y_j^{(t)}\}_{j=1}^m)\)

        \If{\(|\theta^{(t)}-\theta^{(t-1)}|<\varepsilon\)}
            \State \textbf{break}
        \EndIf
    \EndFor

    \State \Return \(\Sigma^\star\leftarrow\Sigma^{(t)}\),
    \(\theta^\star\leftarrow\theta^{(t)}\)
\end{algorithmic}
\end{algorithm}

Here, the factor of \(2\) in the sample gradients has been absorbed into the step size.

\paragraph{Complexity Analysis.}
At each iteration, the main computational cost arises from solving the discrete optimal transport problem between $n$ empirical samples and $m$ Gaussian samples in $\mathbb{R}^d$.
Constructing the pairwise cost matrix requires $O(nmd)$ operations, while the complexity of computing the transport coupling depends on the chosen OT solver.
For example, an entropic OT solver with $K$ Sinkhorn iterations requires $O(Knm)$ operations once the cost matrix is available~\citep{cuturi2013sinkhorn,peyre2019computational}.
Computing the sample gradients costs $O(nmd)$.
Covariance estimation and Gaussian resampling each require $O(md^2)$ operations, while eigendecomposition costs $O(d^3)$.
Thus, when using the Sinkhorn solver, the overall per-iteration complexity is
\[
    O\bigl(nm(d+K)+md^2+d^3\bigr).
\]
For $T$ outer iterations, the total complexity is
\[
    O\bigl(T\bigl(nm(d+K)+md^2+d^3\bigr)\bigr).
\]
For large sample sizes and moderate dimension, the transport-related computations typically dominate the runtime.
We also report runtime in our experiments to illustrate the practical computational cost of the proposed algorithm.

%% file: sec/app_flatness.tex
\newpage
\section{Proof of Flatness of the Filling Cone}
\label{sec:prf_flat_add}

In this section, we prove that the two-dimensional filling cone generated by two rays emanating from the apex \( [\delta_0] \) in the quotient space \( (\Qset, RW_2) \) is flat whenever one of the rays is generated by an absolutely continuous distribution. 
The key observation is that, after a unique optimal alignment ensured by Brenier’s theorem~\citep{Brenier91Unique, villani2003topics}, all Wasserstein distances between any pair of points within the filling cone are induced by a fixed quadratic form in a global affine coordinate system.

\subsection{Setup and Notation}

Let \( [\mu], [\nu] \in \Qset \) and assume that ray $[[\mu]]$ is different from ray $[[\nu]]$.
Without loss of generality, we assume that both representatives, $\mu$ and $\nu$, are centered and normalized to unit \( RW_2 \) norm, i.e.,
\[
\bar{\mu} = \bar{\nu} = 0,
\qquad
\|\mu\|_{RW_2} = \|\nu\|_{RW_2} = 1.
\]
Throughout the proof, we assume that the source distribution \( \mu \) is absolutely continuous.

\subsection{Uniqueness of the Aligned Optimal Map}

\begin{lemma}[Uniqueness of the Optimal Transport Map]
\label{lem:unique}
If \( \mu \) is absolutely continuous, then the optimal transport plan between \( \mu \) and \( \nu \) for the quadratic cost is unique and given by \( (\Id,T)_\#\mu \), where \( T=\nabla\phi \) is the gradient of a convex function.
\end{lemma}

\begin{proof}
This follows directly from Brenier’s theorem~\citep{Brenier91Unique, villani2009OT}.
\end{proof}

\subsection{Geodesics in the Filling Cone}

Using the unique optimal map \(T\), any pair of points, \( \rho_0 \) and \( \rho_1 \), lying on the two rays can be parametrized by scalars \( s_0, s_1 \ge 0 \) as
\[
\rho_0 = (s_0 \Id)_\# \mu,
\qquad
\rho_1 = (s_1 T)_\# \mu .
\]

\begin{lemma}[Form of Geodesics]
\label{lem:geodesic}
The Wasserstein geodesic connecting \( \rho_0 \) and \( \rho_1 \) is given by
\[
\rho_t = (F_t)_\# \mu,
\quad
F_t(x) = (1-t)s_0 x + t s_1 T(x),
\quad
t \in [0,1].
\]
\end{lemma}

\begin{proof}
This follows from McCann’s displacement interpolation, generated by the unique optimal plan \( (\Id,T)_\#\mu \), composed with the endpoint scalings \( x\mapsto s_0 x \) and \( y\mapsto s_1 y \). 
The result also trivially extends to the cases where \( s_0=0 \) or \( s_1=0 \).
\end{proof}

\subsection{Coordinate Chart in the Filling Cone}
Since the filling cone is defined as the union of all Wasserstein geodesics connecting pairs of points lying on the two rays generated by $\mu$ and $\nu$,
any point $\zeta$ in the filling cone must lie on at least one such geodesic.
Consequently, it would have the form of
\[
\zeta = (F_t)_\# \mu,
\qquad
F_t(x) = a x + b T(x),
\]
where $(a,b) = \bigl((1-t)s_0,\, t s_1\bigr)$ for some
$t \in [0,1]$ and $s_0, s_1 \ge 0$ determined by the endpoints
\(
\rho_0 = (s_0 \Id)_\#\mu
\)
and
\(
\rho_1 = (s_1 T)_\#\mu.
\)

Since the interpolation map $F_t$ depends linearly on $x$ and $T(x)$,
the pair $(a, b)$ can be viewed as the coefficients of $(\Id, T)$.
Therefore, any point $\zeta$ in the filling cone can be described by the
coordinates $(a,b)$ with respect to the basis $(\Id, T)$.


Moreover, this pair $(a,b)$ is uniquely determined, and hence the coordinate description is intrinsically well-defined. Indeed, although a given point $\zeta$ in the filling cone may lie on two different geodesics corresponding to distinct choices of $(t,s_0,s_1)$, the uniqueness of the optimal transport map $T$ guarantees that $\zeta$ admits a unique representation of the form
$\zeta = (a\,\Id + b\,T)_\# \mu$.
Consequently, we may define one affine coordinate chart on the filling cone.

Motivated by this observation, for general $a,b \ge 0$ we define push forward $S_{a,b}$ as 
\[
S_{a,b}(x) := a\,x + b\,T(x),
\qquad
\rho_{a,b} := (S_{a,b})_\# \mu.
\]
The filling cone can then be characterized as the image of the map
\[
\Phi : \mathbb{R}^2_{\ge 0} \to \Qset,
\qquad
\Phi(a,b) = \rho_{a,b}.
\]

\subsection{Exact \( W_2 \) Distances inside the Filling Cone}

We begin by recalling two lemmas on optimal transport plans, and subsequently derive the distance formula for any pair of points in the filling cone.

\begin{lemma}[Optimality and $c$-cyclical Monotonicity]
\label{lem:cyclic-opt}
Let $c:\mathbb{R}^d\times\mathbb{R}^d\to\mathbb{R}$ be continuous and let 
$\pi\in\Pi(\mu,\nu)$ be a transport plan. Then $\pi$ is optimal for the cost $c$
if and only if the support is $c$-cyclically monotone, namely for any finite family
\[
(x_1,y_1),\dots,(x_m,y_m)\in\supp(\pi)
\]
one has
\[
\sum_{i=1}^m c(x_i,y_i)
\le
\sum_{i=1}^m c(x_i,y_{i+1}),
\qquad y_{m+1}=y_1.
\]
\end{lemma}

\begin{lemma}[Flip Invariance for Quadratic Cost]
\label{lem:flip}
Let $T:\mathbb{R}^d\to\mathbb{R}^d$ be an optimal transport map for the quadratic cost
$c(x,y)=\|x-y\|^2$. If the graph
\[
\Gamma=\{(x,T(x))\}
\]
is $c$-cyclically monotone, then the flipped set
\[
\Gamma^\top=\{(T(x),x)\}
\]
is also $c$-cyclically monotone.
\end{lemma}

\begin{proof}
For quadratic cost
\[
c(x,y)=\|x-y\|^2
=|x|^2+|y|^2-2\langle x,y\rangle.
\]
Thus $c$-cyclical monotonicity of $\Gamma$ is equivalent to
\[
\sum_{i=1}^m \langle x_i,T(x_i)\rangle
\ge
\sum_{i=1}^m \langle x_i,T(x_{i+1})\rangle .
\]
By cyclic relabeling,
\[
\sum_{i=1}^m \langle T(x_i),x_i\rangle
\ge
\sum_{i=1}^m \langle T(x_i),x_{i+1}\rangle ,
\]
which is exactly $c$-cyclical monotonicity of $\Gamma^\top$.
\end{proof}

\begin{proposition}[Distance Induced by Affine Coordinates]
\label{prop:affine-distance}
For any $(a_1,b_1),(a_2,b_2)\in\mathbb{R}_{\ge0}^2$, let
\[
\rho_{a,b}=(S_{a,b})_\#\mu,
\qquad
S_{a,b}(x)=ax+bT(x).
\]
Then
\[
W_2^2(\rho_{a_1,b_1},\rho_{a_2,b_2})
=
\int \|(a_1-a_2)x+(b_1-b_2)T(x)\|^2\,d\mu.
\]
\end{proposition}

\begin{proof}
Consider the transport plan
\[
\pi:=(S_{a_1,b_1},S_{a_2,b_2})_\#\mu
\in\Pi(\rho_{a_1,b_1},\rho_{a_2,b_2}).
\]
Its transport cost is
\begin{align*}
&\int_{\RR^d\times\RR^d}\|u-v\|^2\,d\pi(u,v)  \\
&= \int_{\RR^d} \|S_{a_1,b_1}(x)-S_{a_2,b_2}(x)\|^2\,d\mu \\
&= \int_{\RR^d} \|(a_1-a_2)x+(b_1-b_2)T(x)\|^2\,d\mu.
\end{align*}

If we can show \(\pi\) is optimal, then the desired identity follows immediately. 
Therefore, it remains to prove that \(\pi\) is optimal.

By Lemma~\ref{lem:cyclic-opt},
it suffices to show that the support of $\pi$
\[
\Gamma
:=
\{(S_{a_1,b_1}(x),S_{a_2,b_2}(x)) : x\in\supp\mu\}
\]
is $c$-cyclically monotone.

Let $x_1,\dots,x_m\in\supp\mu$ and set $x_{m+1} = x_1$. 
Denote $y_i=T(x_i)$ and set $y_{m+1} = y_1$.

Define
\[
u_i=a_1x_i+b_1y_i,
\qquad
v_i=a_2x_i+b_2y_i.
\]
We must show
\[
\sum_{i=1}^m \|u_i-v_i\|^2
\le
\sum_{i=1}^m \|u_i-v_{i+1}\|^2.
\]

Because the cost is quadratic, it is equivalent to show
\[
\sum_{i=1}^m \langle u_{i+1},\, v_{i+1}-v_i\rangle \ge 0.
\]

Since
\[
u_i=a_1x_i+b_1y_i,
\qquad
v_i=a_2x_i+b_2y_i,
\]
and hence
\[
v_{i+1}-v_i
=
a_2(x_{i+1}-x_i)+b_2(y_{i+1}-y_i).
\]
Therefore
\[
\begin{aligned}
&\sum_{i=1}^m \langle u_{i+1},v_{i+1}-v_i\rangle \\
=&
\sum_{i=1}^m
\left\langle a_1x_{i+1}+b_1y_{i+1},\, a_2(x_{i+1}-x_i)+b_2(y_{i+1}-y_i)\right\rangle \\
=&
a_1a_2\Sigma_{xx}+a_1b_2\Sigma_{xy}+b_1a_2\Sigma_{yx}+b_1b_2\Sigma_{yy},
\end{aligned}
\]
where
\[
\Sigma_{xx}:=\sum_{i=1}^m \langle x_{i+1},x_{i+1}-x_i\rangle,
\quad
\Sigma_{xy}:=\sum_{i=1}^m \langle x_{i+1},y_{i+1}-y_i\rangle,
\]
\[
\Sigma_{yx}:=\sum_{i=1}^m \langle y_{i+1},x_{i+1}-x_i\rangle,
\quad
\Sigma_{yy}:=\sum_{i=1}^m \langle y_{i+1},y_{i+1}-y_i\rangle.
\]

We now show that each of these four quantities is nonnegative.

First, for any sequence \(z_1,\dots,z_m\) with \(z_{m+1}=z_1\),
\[
\sum_{i=1}^m \langle z_{i+1},z_{i+1}-z_i\rangle
=
\frac12\sum_{i=1}^m \|z_{i+1}-z_i\|^2 \ge 0.
\]
Applying this with \(z_i=x_i\) and \(z_i=y_i\) yields
\[
\Sigma_{xx}\ge 0,
\qquad
\Sigma_{yy}\ge 0.
\]

Next, since \(T\) is the optimal transport map from \(\mu\) to \(\nu\) for quadratic cost, the graph
\[
\operatorname{Graph}(T)=\{(x,T(x))\}
\]
is \(c\)-cyclically monotone. In the quadratic-cost formulation this means
\[
\sum_{i=1}^m \langle x_{i+1},\, T(x_{i+1})-T(x_i)\rangle \ge 0,
\]
that is,
\[
\Sigma_{xy}\ge 0.
\]

Moreover, because the cost \(c(x,y)=\|x-y\|^2\) is symmetric, based on Lemma~\ref{lem:flip}, the flipped set
\[
\{(T(x),x)\}
\]
is also \(c\)-cyclically monotone. Equivalently,
\[
\sum_{i=1}^m \langle T(x_{i+1}),\, x_{i+1}-x_i\rangle \ge 0,
\]
so
\[
\Sigma_{yx}\ge 0.
\]

Since \(a_1,b_1,a_2,b_2\ge 0\), all coefficients in
\[
a_1a_2\Sigma_{xx}+a_1b_2\Sigma_{xy}+b_1a_2\Sigma_{yx}+b_1b_2\Sigma_{yy}
\]
are nonnegative. Therefore
\[
\sum_{i=1}^m \langle u_{i+1},v_{i+1}-v_i\rangle \ge 0.
\]
This proves that \(\Gamma\) is \(c\)-cyclically monotone, and hence \(\pi\) is optimal.

Finally, the optimal cost is exactly the cost of \(\pi\):
\[
\begin{aligned}
& W_2^2(\rho_{a_1,b_1},\rho_{a_2,b_2}) \\
&=
\int_{\RR^d\times\R^d}\|u-v\|^2\,d\pi(u,v) \\
&=
\int_{\RR^d}\|S_{a_1,b_1}(x)-S_{a_2,b_2}(x)\|^2\,d\mu(x) \\
&=
\int_{\RR^d}\|(a_1-a_2)x+(b_1-b_2)T(x)\|^2\,d\mu(x).
\end{aligned}
\]
This is the claimed formula.
\end{proof}

\subsection{Flatness of the Filling Cone}

\paragraph{Quadratic Form.}
Expanding the square in Proposition~\ref{prop:affine-distance},
\[
\begin{aligned}
W_2^2(\rho_{a_1,b_1},\rho_{a_2,b_2})
&=
(a_1-a_2)^2A + (b_1-b_2)^2 C\\
&\quad
+2(a_1-a_2)(b_1-b_2)B.
\end{aligned}
\]

where $A=\int\|x\|^2\,d\mu,
\quad
B=\int\langle x,T(x)\rangle\,d\mu,
\quad
C=\int\|T(x)\|^2\,d\mu.
$

These constants depend only on $\mu$ and $T$.

\begin{proposition}[Flatness of the Filling Cone]
\label{prop:flat}
The pullback of the \( W_2 \)-metric under \( \Phi(a,b)=(S_{a,b})_\#\mu \) is the constant bilinear form
\[
g =
\begin{pmatrix}
A & B\\
B & C
\end{pmatrix}
\]
in the global affine coordinates \( (a,b) \). Consequently, the induced two-dimensional Riemannian manifold is flat, with zero Gaussian curvature.
\end{proposition}

\begin{proof}
Proposition~\ref{prop:affine-distance} shows that squared distances depend only on the differences \( (a_1-a_2,b_1-b_2) \) through a fixed quadratic form. Hence the metric coefficients are constant in the \( (a,b) \) chart. A two-dimensional Riemannian metric with constant coefficients is isometric (up to a linear change of variables) to a Euclidean cone, and therefore has zero curvature.
\end{proof}

%% file: sec/app_proofs.tex
\newpage
\section{Proof of the Closed-Form Expressions for the One-dimensional Gaussian Family}
\label{sec:prf_one_dim_add}

In this section, we prove Proposition~\ref{prop:1d_closed_form} by deriving closed-form expressions for the orthogonal projection distance \( p \), and the relative Wasserstein angle \( \theta \) in one dimension.

Let
\[
\mu = \frac{1}{n}\sum_{i=1}^{n}\delta_{x_i}, 
\qquad 
\nu = \mathcal{N}(0,\sigma^2),
\]
where \( x_1 \le \cdots \le x_n \) are sorted samples, and consider the quadratic cost
\( c(x,y)=\|x-y\|^2 \).

\paragraph{Step 1: Quantile Formula for \( W_2^2(\mu,\nu) \).}
In one dimension, the quadratic Wasserstein distance admits the quantile representation
\begin{equation}
\label{eq:w2-quantile}
W_2^2(\mu,\nu)
=
\int_0^1
\bigl(F_\mu^{-1}(t) - F_\nu^{-1}(t)\bigr)^2
\, dt .
\end{equation}
For the empirical measure \( \mu \),
\[
F_\mu^{-1}(t)=x_i
\quad \text{for } t\in\Big(\tfrac{i-1}{n},\tfrac{i}{n}\Big],
\]
and for \( \nu=\mathcal N(0,\sigma^2) \),
\[
F_\nu^{-1}(t)=\sigma \Phi^{-1}(t),
\]
where \( \Phi \) denotes the CDF of \( \mathcal N(0,1) \).

Substituting into Equation~\ref{eq:w2-quantile} yields
\begin{equation*}
\label{eq:sum-pieces}
W_2^2(\mu,\nu)
=
\sum_{i=1}^n
\int_{t_{i-1}}^{t_i}
\bigl(x_i-\sigma\Phi^{-1}(t)\bigr)^2
\, dt,
\quad t_i=\tfrac{i}{n}.
\end{equation*}

\paragraph{Step 2: Change of Variables and Integrations.}
Let \( t=\Phi(z) \), so that \( dt=\phi(z)\,dz \), where \( \phi \) is the standard Gaussian density.
Define \( z_i := \Phi^{-1}(i/n) \), with \( z_0=-\infty \), \( z_n=+\infty \).
Then
\[
\int_{t_{i-1}}^{t_i}
\bigl(x_i-\sigma\Phi^{-1}(t)\bigr)^2
\, dt
=
\int_{z_{i-1}}^{z_i}
(x_i-\sigma z)^2\,\phi(z)\,dz .
\]

Expanding and integrating termwise gives
\[
\begin{aligned}
\int_{z_{i-1}}^{z_i}
(x_i-\sigma z)^2\,\phi(z)\,dz
&=
\frac{x_i^2}{n}
-2\sigma x_i\bigl(\phi(z_{i-1})-\phi(z_i)\bigr) \\
&\quad
+\sigma^2\!\bigl[-z\phi(z)+\Phi(z)\bigr]_{z_{i-1}}^{z_i}.
\end{aligned}
\]

Summing over \( i=1,\dots,n \) yields
\begin{equation}
\label{eq:w2-expanded}
W_2^2(\mu,\nu)
=
\frac{1}{n}\sum_{i=1}^n x_i^2
-2\sigma \sum_{i=1}^n x_i\bigl(\phi(z_{i-1})-\phi(z_i)\bigr)
+\sigma^2 .
\end{equation}

\paragraph{Step 3: Computation of the Projection Length.}
Define the empirical scale
\[
\sigma_\mu := \sqrt{\frac{1}{n}\sum_{i=1}^n x_i^2},
\]
and the \emph{projection length}
\begin{equation*}
l
:=
\sum_{i=1}^n
x_i\bigl(\phi(z_{i-1})-\phi(z_i)\bigr).
\end{equation*}
Then Equation~\ref{eq:w2-expanded} can be written as
\begin{equation*}
W_2^2(\mu,\mathcal N(0,\sigma^2))
=
\sigma_\mu^2 - 2\sigma l + \sigma^2 .
\end{equation*}

\paragraph{Step 4: Projection Distance and \( RW_2 \) Angle.}
In the cone geometry of \( ( \Qset, RW_2) \), the projection length \( l \), the projection distance \( p \), and the relative angle \( \theta \) satisfy the Euclidean relations
\[
l = \|[\mu]\|_{RW_2}\cos\theta,
\qquad
\|[\mu]\|_{RW_2}=\sigma_\mu .
\]
Consequently,
\[
\theta
=
\arccos\!\left(\frac{l}{\sigma_\mu}\right),
\qquad
p
=
\sigma_\mu \sin\theta ,
\]
which proves Proposition~\ref{prop:1d_closed_form}.

\newpage
\section{Closed-Form Expressions for Other Location--Scale Distributions}
\label{sec:sol_other_dis_add}

\subsection{A General Formula for One-Dimensional Location--Scale Families}
\label{subsec:general_location_scale_add}

We first present a general formula for projecting a one-dimensional distribution onto a centered location--scale family.
Let \(\mu\) be a zero-mean probability distribution on \(\mathbb{R}\) with finite second moment, and let
\[
    q_\mu(u):=F_\mu^{-1}(u),
    \qquad u\in(0,1),
\]
be its quantile function. Then
\[
    \|[\mu]\|_{RW_2}^2
    =
    \int_0^1 q_\mu(u)^2\,du
    =:
    \sigma_\mu^2 .
\]

Let \(\nu\) be a centered base distribution with quantile function
\[
    q_\nu(u):=F_\nu^{-1}(u),
    \qquad u\in(0,1),
\]
and assume that \(\nu\) has finite nonzero second moment
\[
    s_\nu^2
    :=
    \int_0^1 q_\nu(u)^2\,du
    <\infty .
\]
The associated location--scale ray is given by the family \(\{\nu_\sigma\}_{\sigma>0}\), whose quantile functions satisfy
\[
    F_{\nu_\sigma}^{-1}(u)
    =
    \sigma q_\nu(u).
\]
Define the normalized base direction
\[
    \psi_\nu(u)
    :=
    \frac{q_\nu(u)}{s_\nu}.
\]
Then \(\int_0^1 \psi_\nu(u)^2\,du=1\). The projection length of \([\mu]\) onto the ray generated by \(\nu\) is
\[
    l_\nu
    :=
    \int_0^1 q_\mu(u)\psi_\nu(u)\,du
    =
    \frac{1}{s_\nu}
    \int_0^1 q_\mu(u)q_\nu(u)\,du .
\]
Consequently, the relative Wasserstein angle and the orthogonal projection distance are
\[
    \theta_\nu
    =
    \arccos\!\left(\frac{l_\nu}{\sigma_\mu}\right),
    \qquad
    p_\nu
    =
    \sigma_\mu\sin\theta_\nu
    =
    \sqrt{\sigma_\mu^2-l_\nu^2}.
\]

When \(\mu\) is empirical,
\[
    \mu=\frac{1}{n}\sum_{i=1}^n\delta_{x_i},
    \qquad
    x_1\leq \cdots \leq x_n,
    \qquad
    \sum_{i=1}^n x_i=0,
\]
we have
\[
    q_\mu(u)=x_i,
    \qquad
    u\in(u_{i-1},u_i],
    \qquad
    u_i=\frac{i}{n}.
\]
Thus,
\[
    l_\nu
    =
    \sum_{i=1}^n x_i
    \int_{u_{i-1}}^{u_i}\psi_\nu(u)\,du .
\]

\subsection{Closed-Form Expression for the Uniform Family}
\label{subsec:one_dim_uniform}

We first consider the centered uniform family
\[
    \mathrm{Unif}\!\left[-\frac{\beta}{2},\frac{\beta}{2}\right],
    \qquad
    \beta\geq0.
\]
For the base distribution \(\mathrm{Unif}[-1/2,1/2]\), the quantile function is
\[
    q_{\mathrm{unif}}(u)=u-\frac12,
    \qquad u\in(0,1),
\]
and
\[
    s_{\mathrm{unif}}^2
    =
    \int_0^1 \left(u-\frac12\right)^2\,du
    =
    \frac{1}{12}.
\]
Hence the normalized direction is
\[
    \psi_{\mathrm{unif}}(u)
    =
    \sqrt{12}\left(u-\frac12\right).
\]

The projection length of \([\mu]\) onto the uniform ray is therefore
\begin{align*}
    l_{\mathrm{unif}}
    &=
    \int_0^1 F_\mu^{-1}(u)\psi_{\mathrm{unif}}(u)\,du \notag\\
    &=
    \sqrt{12}
    \sum_{i=1}^n
    x_i
    \int_{u_{i-1}}^{u_i}
    \left(u-\frac12\right)\,du \notag\\
    &=
    \sqrt{12}
    \sum_{i=1}^n
    x_i\,\frac{2i-n-1}{2n^2}.
\label{eq:one_dim_uniform_proj}
\end{align*}
Since
\[
    \sigma_\mu
    =
    \left(\frac{1}{n}\sum_{i=1}^n x_i^2\right)^{1/2},
\]
the relative Wasserstein angle and projection distance are
\begin{equation*}
\label{eq:one_dim_uniform_angle}
    \theta_{\mathrm{unif}}
    =
    \arccos\!\left(\frac{l_{\mathrm{unif}}}{\sigma_\mu}\right),
    \qquad
    p_{\mathrm{unif}}
    =
    \sigma_\mu\sin\theta_{\mathrm{unif}}
    =
    \sqrt{\sigma_\mu^2-l_{\mathrm{unif}}^2}.
\end{equation*}

\begin{proposition}[Closed-Form Projection onto the Uniform Ray]
\label{prop:1d_uniform_closed_form}
Let \(\mu\) be a one-dimensional empirical distribution with zero mean and sorted samples \(x_1\leq \cdots \leq x_n\). Then the projection length, relative Wasserstein angle, and projection distance of \([\mu]\) onto the uniform ray are given by
\[
    l_{\mathrm{unif}}
    =
    \sqrt{12}
    \sum_{i=1}^n
    x_i\,\frac{2i-n-1}{2n^2},
\]
and
\[
    \theta_{\mathrm{unif}}
    =
    \arccos\!\left(\frac{l_{\mathrm{unif}}}{\sigma_\mu}\right),
    \qquad
    p_{\mathrm{unif}}
    =
    \sigma_\mu\sin\theta_{\mathrm{unif}}.
\]
\end{proposition}

\subsection{Closed-Form Expression for the Logistic Family}
\label{subsec:one_dim_logistic}

We next consider the centered logistic family \(\mathrm{Lg}(0,\sigma)\), \(\sigma\geq0\).
The quantile function of the base logistic distribution is
\[
    q_{\mathrm{Lg}}(u)
    =
    \log\!\left(\frac{u}{1-u}\right),
    \qquad u\in(0,1).
\]
Since the variance of the standard logistic distribution is \(\pi^2/3\), we have
\[
    s_{\mathrm{Lg}}=\frac{\pi}{\sqrt{3}},
    \qquad
    \psi_{\mathrm{Lg}}(u)
    =
    \frac{\sqrt{3}}{\pi}
    \log\!\left(\frac{u}{1-u}\right).
\]
Therefore, the projection length is
\begin{align*}
    l_{\mathrm{Lg}}
    &=
    \int_0^1 F_\mu^{-1}(u)\psi_{\mathrm{Lg}}(u)\,du \notag\\
    &=
    \frac{\sqrt{3}}{\pi}
    \sum_{i=1}^n
    x_i
    \int_{u_{i-1}}^{u_i}
    \log\!\left(\frac{u}{1-u}\right)\,du .
\label{eq:logistic_proj}
\end{align*}
Using the antiderivative
\[
    H_{\mathrm{Lg}}(u)
    =
    u\log u+(1-u)\log(1-u),
\]
with the convention \(0\log0=0\), we obtain
\[
    l_{\mathrm{Lg}}
    =
    \frac{\sqrt{3}}{\pi}
    \sum_{i=1}^n
    x_i
    \left[
        H_{\mathrm{Lg}}(u_i)-H_{\mathrm{Lg}}(u_{i-1})
    \right].
\]
Thus,
\begin{equation*}
\label{eq:logistic_angle}
    \theta_{\mathrm{Lg}}
    =
    \arccos\!\left(\frac{l_{\mathrm{Lg}}}{\sigma_\mu}\right),
    \qquad
    p_{\mathrm{Lg}}
    =
    \sigma_\mu\sin\theta_{\mathrm{Lg}}
    =
    \sqrt{\sigma_\mu^2-l_{\mathrm{Lg}}^2}.
\end{equation*}

\subsection{Closed-Form Expression for the Laplace Family}
\label{subsec:one_dim_laplace}

We finally consider the centered Laplace family \(\mathrm{Lpc}(0,b)\), \(b\geq0\).
The quantile function of the base Laplace distribution with scale \(b=1\) is
\[
    q_{\mathrm{Lpc}}(u)
    =
    \begin{cases}
        \log(2u), & 0<u\leq \frac12,\\[4pt]
        -\log\!\bigl(2(1-u)\bigr), & \frac12<u<1.
    \end{cases}
\]
Since the variance of \(\mathrm{Lpc}(0,1)\) is \(2\), we have
\[
    s_{\mathrm{Lpc}}=\sqrt{2},
    \qquad
    \psi_{\mathrm{Lpc}}(u)
    =
    \frac{1}{\sqrt{2}}q_{\mathrm{Lpc}}(u).
\]
Therefore, the projection length is
\begin{align}
    l_{\mathrm{Lpc}}
    &=
    \int_0^1 F_\mu^{-1}(u)\psi_{\mathrm{Lpc}}(u)\,du \notag\\
    &=
    \frac{1}{\sqrt{2}}
    \sum_{i=1}^n
    x_i
    \int_{u_{i-1}}^{u_i}
    q_{\mathrm{Lpc}}(u)\,du .
\label{eq:laplace_proj}
\end{align}

The required antiderivative is piecewise. Define
\[
    H_{\mathrm{Lpc}}(u)
    =
    \begin{cases}
        u\log(2u)-u, & 0<u\leq \frac12,\\[4pt]
        (1-u)\log\!\bigl(2(1-u)\bigr)-(1-u), & \frac12<u<1.
    \end{cases}
\]
Then each summand in Equation~\ref{eq:laplace_proj} is obtained from the interval endpoints:
\[
    \int_{u_{i-1}}^{u_i}
    q_{\mathrm{Lpc}}(u)\,du
    =
    H_{\mathrm{Lpc}}(u_i)-H_{\mathrm{Lpc}}(u_{i-1}),
\]
with the interval split at \(u=1/2\) whenever necessary.

Consequently,
\[
    \theta_{\mathrm{Lpc}}
    =
    \arccos\!\left(\frac{l_{\mathrm{Lpc}}}{\sigma_\mu}\right),
    \qquad
    p_{\mathrm{Lpc}}
    =
    \sigma_\mu\sin\theta_{\mathrm{Lpc}}
    =
    \sqrt{\sigma_\mu^2-l_{\mathrm{Lpc}}^2}.
\]

%% file: sec/app_mm_vs_nearest.tex
\newpage
\section{The Moment-Matching Gaussian Is Not the $W_2$-Nearest Gaussian}
\label{sec:MM_Gaussian}

In this section, we explain why the moment-matching Gaussian is \textit{not} the \(W_2\)-nearest Gaussian approximation to a non-Gaussian distribution.

Let $\mu\in\PR{2}{d}$ have mean $\bar{\mu}$ and covariance matrix $\Sigma_\mu$, and let
\[
\mu_{\mathrm{MM}}
=
\mathcal{N}(\bar{\mu},\Sigma_\mu)
\]
denote its moment-matching Gaussian approximation.

On the one hand, the squared length of $[\mu]$ is given by
\[
\|[\mu]\|_{RW_2}^2
=
RW_2^2([\delta_0],[\mu])
=
\inf_{z\in\mathbb{R}^d}
\int_{\mathbb{R}^d}\|x-z\|_2^2\,d\mu(x).
\]
The minimizer is $z=\bar{\mu}$, and hence
\[
\|[\mu]\|_{RW_2}^2
=
\int_{\mathbb{R}^d}\|x-\bar{\mu}\|_2^2\,d\mu(x)
=
\operatorname{Tr}(\Sigma_\mu).
\]

On the other hand, the squared length of $[\mu_{\mathrm{MM}}]$ is given by
\[
\|[\mu_{\mathrm{MM}}]\|_{RW_2}^2
=
RW_2^2([\delta_0],[\mu_{\mathrm{MM}}])
=
W_2^2(\delta_0,\mathcal{N}(0,\Sigma_\mu)).
\]
Since the squared \(W_2\) distance from a Dirac measure to a probability measure equals the second moment of that measure about the Dirac location, we have
\[
W_2^2(\delta_0,\mathcal{N}(0,\Sigma_\mu))
=
\int_{\mathbb{R}^d}\|x\|_2^2\,d\mathcal{N}(0,\Sigma_\mu)(x)
=
\operatorname{Tr}(\Sigma_\mu).
\]
Thus,
\[
\|[\mu]\|_{RW_2}
=
\|[\mu_{\mathrm{MM}}]\|_{RW_2}
=
\sqrt{\operatorname{Tr}(\Sigma_\mu)}.
\]

Now consider the Gaussian ray containing \([\mu_{\mathrm{MM}}]\), and let \(\theta\) be the \(RW_2\) angle between this ray and the ray generated by \([\mu]\). 
Set 
\[ 
r := \|[\mu]\|_{RW_2} = \|[\mu_{\mathrm{MM}}]\|_{RW_2}. 
\]
In the Euclidean filling cone spanned by the two rays, the orthogonal projection of $[\mu]$ onto the Gaussian ray has length
\[
l
=
r\cos\theta .
\]
Since $\mu$ is non-Gaussian, its ray differs from the
Gaussian ray containing $[\mu_{\mathrm{MM}}]$.
Hence $\theta>0$ and $\cos\theta<1$, so
\[
l
=
r\cos\theta
<
r
=
\|[\mu_{\mathrm{MM}}]\|_{RW_2}.
\]
This implies that the projection point of $[\mu]$ onto the Gaussian ray differs from the moment-matching Gaussian, since its length is strictly smaller than $\|[\mu_{\mathrm{MM}}]\|_{RW_2}$.
Since the projection point is a Gaussian distribution that is closer to $\mu$ than the moment-matching Gaussian is, the moment-matching Gaussian is different from the $W_2$-nearest Gaussian approximation of $\mu$.

%% file: sec/app_others.tex
\newpage
\section{Coordinate-Separable Distributions in the PCA Basis}
\label{sec:sol_Coordinate_add}
Let
$
\mu=\frac1n\sum_{i=1}^n\delta_{x_i}
$
be an empirical distribution with mean \(\bar\mu\) and covariance matrix
$$
\Sigma_\mu
=
U\operatorname{diag}(\sigma_1^2,\ldots,\sigma_d^2)U^\top,
$$
where \(U\) is orthogonal and \(\sigma_k^2>0\) for
\(k=1,\ldots,d\).

Define the PCA coordinates and their marginal distributions by
$$
y_i:=U^\top(x_i-\bar\mu),
\qquad
\mu_k:=\frac1n\sum_{i=1}^n\delta_{y_{i,k}}.
$$
We say that \(\mu\) is \emph{coordinate-separable in the PCA basis} if
$$
\bigl(U^\top(\,\cdot-\bar\mu)\bigr)_\#\mu
=
\bigotimes_{k=1}^d\mu_k.
$$
For each coordinate \(k\), sort the samples such that
$$
y_{1,k}\leq\cdots\leq y_{n,k},
$$
and define
$$
l_k
:=
\sum_{i=1}^n y_{i,k}
\bigl(\phi(z_{i-1})-\phi(z_i)\bigr),
\qquad
z_i:=\Phi^{-1}(i/n),
$$
where \(z_0=-\infty\), \(z_n=+\infty\), and \(\phi\) and \(\Phi\)
denote the standard normal density and cumulative distribution functions, respectively.

Consider any Gaussian distribution \(\nu\) with mean \(\bar\mu\), and let \(s_k^2\) denote its \(k\)-th marginal variance in the PCA basis. By the one-dimensional formula,
$$
W_2^2\!\left(\mu_k,\mathcal N(0,s_k^2)\right)
=
\sigma_k^2-2s_kl_k+s_k^2
=
\sigma_k^2-l_k^2+(s_k-l_k)^2.
$$

Hence, the coordinate-wise distance is minimized at
$
s_k=l_k.
$

Moreover, any coupling of the full distributions induces, under coordinate projection, a coupling between each pair of corresponding marginals. Since the squared Euclidean transport cost decomposes as the sum of the coordinate-wise costs, we obtain
$$
W_2^2(\mu,\nu)
\geq
\sum_{k=1}^d
W_2^2\!\left(\mu_k,\mathcal N(0,s_k^2)\right)
\geq
\sum_{k=1}^d(\sigma_k^2-l_k^2).
$$
This lower bound is attained by choosing the Gaussian whose coordinates are independent in the PCA basis with \(s_k=l_k\). Indeed, under the coordinate-separability assumption, the product of the one-dimensional optimal couplings is then an optimal coupling for the full distributions.

Since a \(W_2\)-nearest Gaussian must have the same mean as \(\mu\), a \(W_2\)-nearest Gaussian is therefore
$$
\nu^\star
=
\mathcal N\!\left(
\bar\mu,\,
U\operatorname{diag}(l_1^2,\ldots,l_d^2)U^\top
\right).
$$

Consequently, the minimum \(RW_2\) angle to the Gaussian family and the corresponding projection distance are
$$
\theta^\star
=
\arccos\!\sqrt{
\frac{\sum_{k=1}^d l_k^2}
     {\sum_{k=1}^d \sigma_k^2}},
\qquad
p^\star
=
\sqrt{\sum_{k=1}^d(\sigma_k^2-l_k^2)}.
$$

\newpage
\section{An Obstruction to Extending Pairwise Flatness to Three Rays}
\label{sec:three_rays_add}
We present an example in which each pair of three
distributional rays determines a unique flat filling cone,
but the union of these cones, equipped with the ambient
$RW_2$ distance, admits no Euclidean isometric embedding.

\paragraph{Setup.}
Let
\[
\gamma:=\mathcal N(0,1)\otimes\delta_0
\]
be the standard one-dimensional Gaussian distribution
supported on the $x$-axis in $\mathbb R^2$.
Let $R_\theta$ denote rotation about the origin by angle
$\theta$, and define
\[
\mu_1:=\gamma,\qquad
\mu_2:=(R_{\pi/3})_\#\gamma,\qquad
\mu_3:=(R_{2\pi/3})_\#\gamma.
\]
These are centered, degenerate Gaussian distributions
supported on three lines. Each has unit $RW_2$ norm.
Because all distributions considered below are centered,
their $RW_2$ distances equal their $W_2$ distances.

\paragraph{Pairwise Optimal Maps.}
For unit vectors $u,v\in\mathbb R^2$ and $s,t\geq0$,
the quadratic Wasserstein distance between the
corresponding line-supported Gaussians is
\begin{equation}
\label{eq:three_rays_gaussian_distance}
W_2^2\!\left(
\mathcal N(0,s^2uu^\top),
\mathcal N(0,t^2vv^\top)
\right)
=
s^2+t^2-2st|u^\top v|.
\end{equation}
Indeed, let $\pi$ be a coupling of two standard normal
distributions, where $z,w\in\mathbb R$ denote the scalar
coordinates parameterizing points $szu$ and $twv$ on
the two supporting lines. Their transport cost is
\[
\int\|szu-twv\|^2\,d\pi(z,w)
=
s^2+t^2-2st(u^\top v)\int zw\,d\pi(z,w).
\]

By the Cauchy--Schwarz inequality and the unit second
moments of both marginals, we have
\[
\left|\int zw\,d\pi(z,w)\right|
\leq
\left(\int z^2\,d\pi(z,w)\right)^{1/2}
\left(\int w^2\,d\pi(z,w)\right)^{1/2}
=1.
\]

The minimum is attained by $w=z$ when $u^\top v>0$
and by $w=-z$ when $u^\top v<0$.
For $s,t>0$ and $u^\top v\neq0$, the equality condition
also establishes uniqueness of the optimal coupling.

Consequently, the unique optimal maps between our
three line-supported distributions, $\{\mu_1, \mu_2, \mu_3\}$, restricted to their supporting lines,
are
\[
T_{1\to2}=R_{\pi/3},\qquad
T_{2\to3}=R_{\pi/3},\qquad
T_{1\to3}=R_{-\pi/3}.
\]
The last map is valid because rotations by $-\pi/3$
and $2\pi/3$ produce the same centered Gaussian
distribution on the target line, but the former
gives the smaller transport cost.

Equation~\ref{eq:three_rays_gaussian_distance} yields
\[
W_2^2(\mu_i,\mu_j)=1,\qquad i\neq j.
\]
Since each measure has unit norm, all three pairwise
$RW_2$ angles equal $\pi/3$.

\paragraph{Pairwise Flatness.}
The measures above are singular in $\mathbb R^2$,
so we verify flatness directly.

For each pair, orient the two supporting lines according
to its optimal map. The oriented unit vectors $u,v$
then satisfy $u^\top v=1/2$.
The displacement interpolations between dilations of
the two measures, including their boundary limits,
have the form
\[
\mathcal N\!\left(0,(au+bv)(au+bv)^\top\right),
\qquad a,b\geq0.
\]
The vectors $au+bv$ form a Euclidean sector of angle
$\pi/3$. Any two vectors $x,y$ in this sector have
$x^\top y\geq0$, so
Equation~\ref{eq:three_rays_gaussian_distance} gives
\[
W_2^2\!\left(
\mathcal N(0,xx^\top),
\mathcal N(0,yy^\top)
\right)
=
\|x-y\|^2.
\]
Thus each closed pairwise filling cone is isometric
to a Euclidean sector and is therefore flat.

\paragraph{Failure of a Common Euclidean Embedding.}
Let $\nu_{12}$ and $\nu_{23}$ be the geodesic midpoints
between $\mu_1,\mu_2$ and $\mu_2,\mu_3$, respectively:
\[
\nu_{12}
:=
\left(\frac{\Id+R_{\pi/3}}{2}\right)_\#\mu_1,
\qquad
\nu_{23}
:=
\left(\frac{\Id+R_{\pi/3}}{2}\right)_\#\mu_2.
\]
Using the rotation identity
\[
\frac{\Id+R_{\pi/3}}{2}
=
\frac{\sqrt3}{2}R_{\pi/6},
\]
we obtain
\[
\nu_{12}
=
\left(\frac{\sqrt3}{2}R_{\pi/6}\right)_\#\gamma,
\qquad
\nu_{23}
=
\left(\frac{\sqrt3}{2}R_{\pi/2}\right)_\#\gamma.
\]
Both measures have squared norm $3/4$, and their
supporting lines make an angle of $\pi/3$.
Therefore,
\begin{equation}
\label{eq:three_rays_midpoint_distance}
W_2^2(\nu_{12},\nu_{23})
=
\frac34+\frac34
-2\frac34\cos\frac{\pi}{3}
=
\frac34.
\end{equation}

Suppose that the union of the three pairwise filling
cones admits an isometric embedding into a Euclidean
space, and let $q_i$ be the image of $\mu_i$.
An isometry preserves metric midpoints, so the images
of $\nu_{12}$ and $\nu_{23}$ must be
$(q_1+q_2)/2$ and $(q_2+q_3)/2$, respectively.
Their squared distance would then equal
\[
\left\|
\frac{q_1+q_2}{2}
-
\frac{q_2+q_3}{2}
\right\|^2
=
\frac14\|q_1-q_3\|^2
=
\frac14 W_2^2(\mu_1,\mu_3)
=
\frac14,
\]
contradicting
Equation~\ref{eq:three_rays_midpoint_distance}.
Hence no such Euclidean isometric embedding exists.


This example shows that flatness of every pairwise
filling cone does not guarantee a common Euclidean
geometry when three rays are considered together.
It does not rule out flatness for special
triples of rays.

\newpage
\section{One Example where the Nearest Gaussian Does Not Share Eigenvalues or Eigenvectors}
\label{sec:not_sharing_add}

This section gives a concrete example showing that the $W_2$-nearest Gaussian
approximation of an empirical measure may share neither the same eigenvalues
nor the same eigenvectors with the moment-matching Gaussian.


\subsection{A Two-dimensional Counterexample}
\label{subsec:not_same_eigvecs}

Consider the empirical distribution supported on three points
\[
\gamma \;=\; \tfrac13\bigl(\delta_{(0,0)}+\delta_{(4,0)}+\delta_{(0,3)}\bigr),
\]
with mean
\[
\bar\gamma=\Bigl(\tfrac43,\,1\Bigr).
\]
Define the centered and $RW_2$-normalized distribution
\[
\mu
\;:=\;
\frac{\gamma-\bar\gamma}{\|\gamma-\bar\gamma\|_{RW_2}},
\]
so that $\bar\mu=0$ and $\|\mu\|_{RW_2}=1$.

\paragraph{Empirical Covariance and Moment Matching.}
A direct computation shows that the empirical covariance of $\mu$ is
\[
\Sigma_\mu
= \frac{1}{50}
\begin{pmatrix}
32 & -12\\
-12 & 18
\end{pmatrix}.
\]
which has trace one and a nonzero off-diagonal entry.
\paragraph{Eigenvectors.}
In two dimensions, the symmetric covariance matrix \(\Sigma_\mu\) admits an eigendecomposition of the form
$$
\Sigma_\mu
=
R(\omega_{\mathrm{MM}})\,\Lambda\,R(\omega_{\mathrm{MM}})^\top,
\qquad
R(\omega)=
\begin{pmatrix}
\cos\omega & -\sin\omega\\
\sin\omega & \cos\omega
\end{pmatrix},
$$
where \(\Lambda\) is diagonal and \(\omega_{\mathrm{MM}}\in[0,\pi)\) determines the (principal axis) orientation of the eigenvectors. The corresponding rotation angle can be obtained from
$$
\omega_{\mathrm{MM}}
=
\frac{1}{2}
\operatorname{atan2}\!\left(
2\Sigma_{\mu,12},
\Sigma_{\mu,11}-\Sigma_{\mu,22}
\right)
\quad (\mathrm{mod}\ \pi).
$$
Substituting the entries of \(\Sigma_\mu\) gives
$$
\omega_{\mathrm{MM}} \approx 0.8340\,\pi.
$$
To determine the eigenvector orientation of the \(W_2\)-nearest Gaussian, we perform a brute-force search over
$$
\lambda\in(0,1),
\qquad
\omega\in[0,\pi),
$$
using the parameterization
$$
\Sigma(\lambda,\omega)
=
R(\omega)\,
\operatorname{diag}(\lambda,1-\lambda)\,
R(\omega)^\top,
$$
which represents zero-mean Gaussian covariance directions normalized to unit trace. Using Monte Carlo sampling together with exact discrete optimal transport, we obtain a minimizing Gaussian with rotation angle
$$
\omega_\star \approx 0.7983\,\pi.
$$
Since
$$
\omega_\star \neq \omega_{\mathrm{MM}},
$$
the \(W_2\)-nearest Gaussian and the moment-matching Gaussian have different eigenvector orientations.

\paragraph{Eigenvalues.}
The discrepancy between the two Gaussian approximations is not limited to their eigenvectors. After normalizing both covariance matrices to have unit trace, their corresponding eigenvalues are

$$
\begin{aligned}
\operatorname{Spec}(\text{moment-matching Gaussian})
&=\{\,0.7778,\;0.2222\,\}, \\
\operatorname{Spec}(\text{\(W_2\)-nearest Gaussian})
&=\{\,0.9411,\;0.0589\,\}.
\end{aligned}
$$

Thus, the \(W_2\)-nearest Gaussian is substantially more anisotropic than the moment-matching Gaussian, demonstrating that the two approximations differ not only in their principal directions but also in their covariance spectra.


\paragraph{Energy Landscape Visualization.}
Figure~\ref{fig:rw2_energy_landscape} visualizes the $RW_2$ energy landscape
\[
(\lambda,\omega)
\;\longmapsto\;
W_2^2\!\bigl(\mu,\mathcal N(0,\Sigma(\lambda,\omega))\bigr)
\]
over Gaussian directions with unit trace.
The moment-matching Gaussian and the $RW_2$-nearest Gaussian are marked explicitly.
The figure shows that the moment-matching direction is not a local minimizer of the
$RW_2$ objective and that a different combination of eigenvalues and eigenvectors
achieves a strictly smaller transport cost.

\begin{figure}[h]
    \centering
    \includegraphics[width=0.8\textwidth]{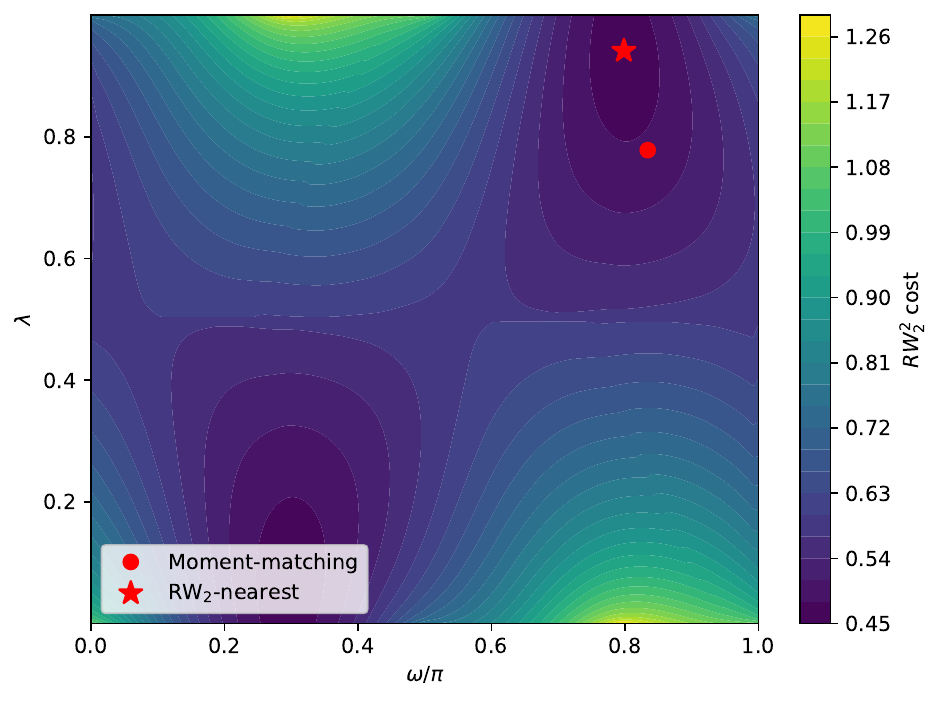}
    \caption{
    $RW_2$ energy landscape over Gaussian directions.
    The contour plot shows the value of
    $W_2^2\!\bigl(\mu,\mathcal N(0,\Sigma(\lambda,\omega))\bigr)$
    as a function of the eigenvalue parameter $\lambda$ and the rotation angle $\omega$.
    The red dot denotes the moment-matching Gaussian, while the red star denotes the
    $RW_2$-nearest Gaussian.
    The latter achieves a strictly smaller $RW_2$ distance, demonstrating that the
    moment-matching Gaussian does not lie on the nearest Gaussian ray.
    }
    \label{fig:rw2_energy_landscape}
\end{figure}


%% file: sec/app_exp.tex
\newpage
\section{Supplementary Experimental Results}
\label{app:exp_add}

\subsection{Visualizations in the One-Dimensional Case}
\label{app:exp_rw2_1d}

We visualize the difference between the \(W_2\)-nearest Gaussian and the moment-matching Gaussian using the two-component Gaussian mixture introduced in Section~\ref{sec:exp1}.
For each separation parameter \(m_c\in\{1,2,4\}\), we draw \(N=256\) samples to construct an empirical distribution.
We compute its \(W_2\)-nearest Gaussian using the closed-form expressions and compare it with the moment-matching Gaussian.
The standard deviations of the moment-matching Gaussian distributions are given by the empirical standard deviations.

\begin{figure}[ht] 
    \centering \includegraphics[width=0.8\textwidth]{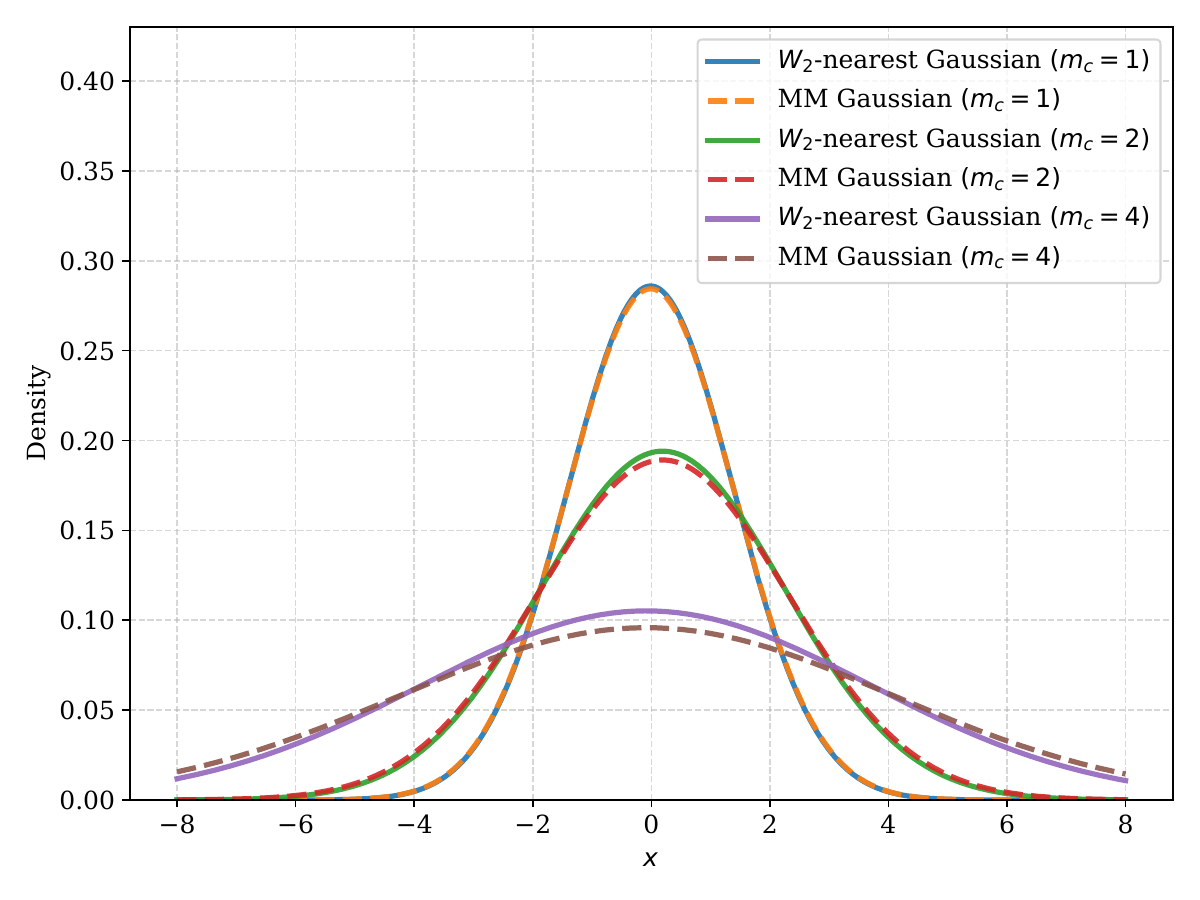} 
    \caption{ Comparison between the \(W_2\)-nearest Gaussian and the moment-matching Gaussian for one-dimensional Gaussian mixture model with different separation parameters \(m_c\). Solid curves represent the \(W_2\)-nearest Gaussian, while dashed curves represent the moment-matching Gaussian. The two Gaussian distributions are close when the mixture components are weakly separated. However, as \(m_c\) increases, the discrepancy becomes more visible. This illustrates that the moment-matching Gaussian is not necessarily the \(W_2\)-nearest Gaussian, even in one dimension. } \label{fig:gaussian_comparison_1d}
\end{figure} 

Figure~\ref{fig:gaussian_comparison_1d} shows that the two Gaussian distributions are close to each other when the mixture components are weakly separated. However, as \(m_c\) increases, the discrepancy becomes more pronounced. This illustrates that the moment-matching Gaussian is not necessarily the \(W_2\)-nearest Gaussian, even in one dimension.

\newpage
\subsection{Non-Gaussianity Evaluation of the Datasets}
\label{sec:exp_add}

\paragraph{Setup.}
We evaluate both raw and latent data representations.
Specifically, we use MNIST~\citep{lecun1994gradient} and EMNIST~\citep{cohen2017EMNIST} for evaluation in pixel space, and CIFAR-10~\citep{krizhevsky2009learning} and CINIC-10~\citep{darlow2018cinic10} for evaluation in feature space, with features extracted by a pretrained Inception-v3 encoder~\citep{szegedy2015rethinking}.

For each dataset, we consider the full empirical distribution, individual classes, and selected mixtures of classes.
We use Algorithm~\ref{alg:rw2_hd} to estimate the \(RW_2\) angle between each distribution and its \(W_2\)-nearest Gaussian approximation, with the moment-matching Gaussian approximation as a baseline.
We report the mean and standard deviation of the estimated angles and running times.
Each experiment is repeated five times with different random seeds, using a base seed of \(0\).
We also track the \(RW_2\) angle over optimization iterations to assess the empirical convergence and stability of the proposed algorithm.

We compute exact discrete optimal transport using \texttt{ot.emd} from the POT library~\citep{flamary2021pot}.
To improve computational efficiency and stability, we apply PCA to reduce each representation to \(d=32\) dimensions.
After PCA, we center each sampled dataset and rescale it to have unit empirical second moment. We use $n=m=512$ and $\eta=128$, corresponding to an effective particle step size of $\eta/m=0.25$. We run 50 iterations and disable tolerance-based early stopping.

\begin{figure}[ht]
    \centering
    \includegraphics[width=0.8\linewidth]{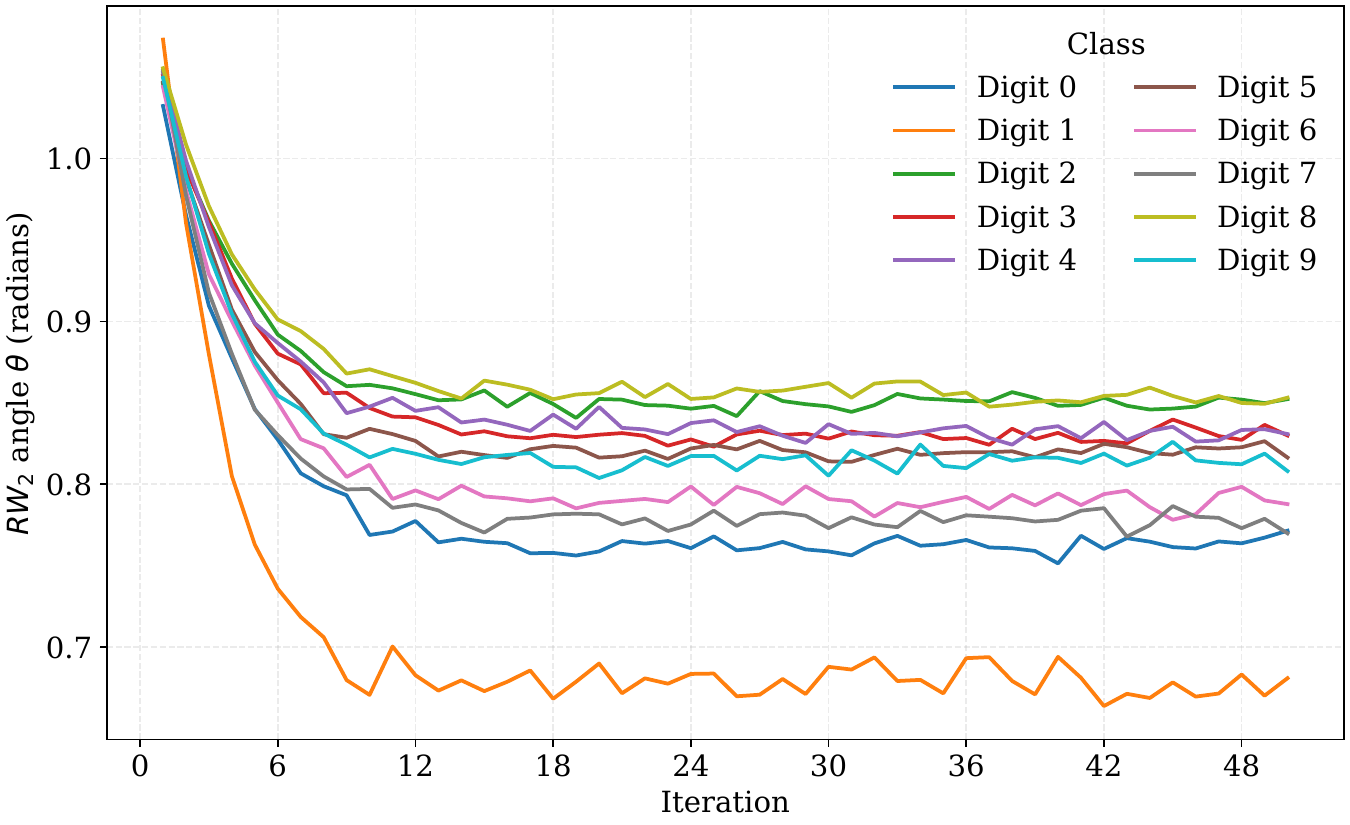}
    \caption{
    Evolution of the estimated \(RW_2\) angle between each MNIST digit-class distribution and its Gaussian approximation during the first experimental repeat.
    Each curve corresponds to one digit class (0--9).
    The angles decrease rapidly during the initial iterations and subsequently stabilize around class-dependent values, providing empirical evidence of the convergence and stability of Algorithm~\ref{alg:rw2_hd}.
    }
    \label{fig:mnist_rw2_convergence}
\end{figure}

\paragraph{Convergence and Stability of the Algorithm.}
Figure~\ref{fig:mnist_rw2_convergence} shows the evolution of the estimated \(RW_2\) angle between each MNIST digit-class distribution and its Gaussian approximation during the first experimental repeat.
Each curve corresponds to one digit class (0--9).
The angles decrease rapidly during the initial iterations and subsequently stabilize around class-dependent values, providing empirical evidence of the convergence and stability of Algorithm~\ref{alg:rw2_hd}.

\begin{table}[ht]
    \centering
    \caption{Non-Gaussianity evaluation of MNIST after PCA reduction to \(d=32\) dimensions. We report the mean and standard deviation of the \(RW_2\) angles between each empirical distribution and its estimated \(W_2\)-nearest and moment-matching Gaussian approximations, together with the running time. All angles are measured in radians.}
    \label{tab:mnist_heterogeneity_d32}
    \resizebox{\textwidth}{!}{
    \begin{tabular}{llcccc}
        \toprule
        Distribution & Classes & MM angle & $W_2$-nearest angle & MM time / sec & $W_2$-nearest time / sec \\
        \midrule
        Full dataset & all & \(0.9757 \pm 0.0032\) & \(0.9287 \pm 0.0040\) & \( 0.0403 \) & \( 3.2181 \) \\
        Digit 0 & 0 & \(0.8273 \pm 0.0013\) & \(0.7694 \pm 0.0037\) & \( 0.0312 \) & \( 2.9139 \) \\
        Digit 1 & 1 & \(0.7261 \pm 0.0144\) & \(0.6675 \pm 0.0093\) & \( 0.0274 \) & \( 2.6568 \) \\
        Digit 2 & 2 & \(0.9005 \pm 0.0038\) & \(0.8490 \pm 0.0065\) & \( 0.0288 \) & \( 2.9011 \) \\
        Digit 3 & 3 & \(0.8868 \pm 0.0037\) & \(0.8293 \pm 0.0019\) & \( 0.0285 \) & \( 2.8844 \) \\
        Digit 4 & 4 & \(0.8816 \pm 0.0059\) & \(0.8294 \pm 0.0058\) & \( 0.0299 \) & \( 2.8641 \) \\
        Digit 5 & 5 & \(0.8715 \pm 0.0084\) & \(0.8146 \pm 0.0055\) & \( 0.0295 \) & \( 2.8737 \) \\
        Digit 6 & 6 & \(0.8422 \pm 0.0042\) & \(0.7874 \pm 0.0065\) & \( 0.0284 \) & \( 2.8429 \) \\
        Digit 7 & 7 & \(0.8357 \pm 0.0054\) & \(0.7789 \pm 0.0065\) & \( 0.0285 \) & \( 2.8447 \) \\
        Digit 8 & 8 & \(0.9046 \pm 0.0071\) & \(0.8528 \pm 0.0043\) & \( 0.0298 \) & \( 2.9332 \) \\
        Digit 9 & 9 & \(0.8663 \pm 0.0090\) & \(0.8073 \pm 0.0040\) & \( 0.0289 \) & \( 2.8845 \) \\
        Mixture & 0,1 & \(0.8460 \pm 0.0042\) & \(0.7971 \pm 0.0063\) & \( 0.0287 \) & \( 2.8176 \) \\
        Mixture & 0,1,2,3 & \(0.9281 \pm 0.0061\) & \(0.8738 \pm 0.0072\) & \( 0.0288 \) & \( 2.8651 \) \\
        Mixture & 0,1,2,3,4,5,6,7 & \(0.9679 \pm 0.0033\) & \(0.9203 \pm 0.0037\) & \( 0.0289 \) & \( 2.9212 \) \\
        Mixture & 3,4 & \(0.9147 \pm 0.0059\) & \(0.8644 \pm 0.0068\) & \( 0.0298 \) & \( 2.9211 \) \\
        Mixture & 3,4,5,6 & \(0.9483 \pm 0.0030\) & \(0.8983 \pm 0.0043\) & \( 0.0294 \) & \( 2.9114 \) \\
        Mixture & 1,2,3,4,5,6,7,8 & \(0.9755 \pm 0.0031\) & \(0.9275 \pm 0.0028\) & \( 0.0298 \) & \( 2.9358 \) \\
        Mixture & 5,6 & \(0.9014 \pm 0.0058\) & \(0.8487 \pm 0.0054\) & \( 0.0294 \) & \( 2.8877 \) \\
        Mixture & 5,6,7,8 & \(0.9582 \pm 0.0028\) & \(0.9066 \pm 0.0031\) & \( 0.0293 \) & \( 2.9299 \) \\
        Mixture & 2,3,4,5,6,7,8,9 & \(0.9737 \pm 0.0072\) & \(0.9320 \pm 0.0030\) & \( 0.0291 \) & \( 2.9416 \) \\
        \bottomrule
    \end{tabular}
    }
\end{table}

\begin{table}[ht]
    \centering
    \caption{Non-Gaussianity evaluation of EMNIST after PCA reduction to \(d=32\) dimensions. We report the mean and standard deviation of the \(RW_2\) angles between each empirical distribution and its estimated \(W_2\)-nearest and moment-matching Gaussian approximations, together with the running time. All angles are measured in radians.}
    
    \label{tab:emnist_heterogeneity_d32}
    \resizebox{\textwidth}{!}{
    \begin{tabular}{llcccc}
        \toprule
        Distribution & Classes & MM angle & $W_2$-nearest angle & MM time / sec & $W_2$-nearest time / sec \\
        \midrule
        Full dataset & all & \(0.9693 \pm 0.0012\) & \(0.9211 \pm 0.0009\) & \( 0.0451 \) & \( 3.1928 \) \\
        Digit 0 & 0 & \(0.8275 \pm 0.0068\) & \(0.7683 \pm 0.0047\) & \( 0.0317 \) & \( 3.0877 \) \\
        Digit 1 & 1 & \(0.7308 \pm 0.0122\) & \(0.6733 \pm 0.0081\) & \( 0.0299 \) & \( 2.9363 \) \\
        Digit 2 & 2 & \(0.9111 \pm 0.0056\) & \(0.8588 \pm 0.0040\) & \( 0.0340 \) & \( 3.1503 \) \\
        Digit 3 & 3 & \(0.8948 \pm 0.0062\) & \(0.8406 \pm 0.0038\) & \( 0.0287 \) & \( 3.1221 \) \\
        Digit 4 & 4 & \(0.8990 \pm 0.0042\) & \(0.8523 \pm 0.0043\) & \( 0.0319 \) & \( 3.1737 \) \\
        Digit 5 & 5 & \(0.8697 \pm 0.0052\) & \(0.8111 \pm 0.0040\) & \( 0.0310 \) & \( 3.0895 \) \\
        Digit 6 & 6 & \(0.8563 \pm 0.0048\) & \(0.8062 \pm 0.0023\) & \( 0.0328 \) & \( 3.1696 \) \\
        Digit 7 & 7 & \(0.8420 \pm 0.0047\) & \(0.7937 \pm 0.0037\) & \( 0.0319 \) & \( 3.1014 \) \\
        Digit 8 & 8 & \(0.9236 \pm 0.0071\) & \(0.8738 \pm 0.0072\) & \( 0.0345 \) & \( 3.1455 \) \\
        Digit 9 & 9 & \(0.8686 \pm 0.0084\) & \(0.8150 \pm 0.0049\) & \( 0.0281 \) & \( 3.0975 \) \\
        Mixture & 0,1 & \(0.8321 \pm 0.0080\) & \(0.7845 \pm 0.0091\) & \( 0.0310 \) & \( 3.0872 \) \\
        Mixture & 0,1,2,3 & \(0.9259 \pm 0.0059\) & \(0.8733 \pm 0.0078\) & \( 0.0290 \) & \( 2.8607 \) \\
        Mixture & 0,1,2,3,4,5,6,7 & \(0.9617 \pm 0.0050\) & \(0.9171 \pm 0.0041\) & \( 0.0300 \) & \( 2.9072 \) \\
        Mixture & 3,4 & \(0.9202 \pm 0.0067\) & \(0.8739 \pm 0.0100\) & \( 0.0285 \) & \( 2.9214 \) \\
        Mixture & 3,4,5,6 & \(0.9497 \pm 0.0050\) & \(0.9071 \pm 0.0055\) & \( 0.0288 \) & \( 2.9048 \) \\
        Mixture & 1,2,3,4,5,6,7,8 & \(0.9712 \pm 0.0061\) & \(0.9278 \pm 0.0051\) & \( 0.0296 \) & \( 2.9163 \) \\
        Mixture & 5,6 & \(0.9230 \pm 0.0048\) & \(0.8720 \pm 0.0080\) & \( 0.0298 \) & \( 2.8769 \) \\
        Mixture & 5,6,7,8 & \(0.9502 \pm 0.0047\) & \(0.9032 \pm 0.0029\) & \( 0.0290 \) & \( 2.8985 \) \\
        Mixture & 2,3,4,5,6,7,8,9 & \(0.9686 \pm 0.0037\) & \(0.9230 \pm 0.0035\) & \( 0.0287 \) & \( 2.9132 \) \\
        \bottomrule
    \end{tabular}
    }
\end{table}

\paragraph{Experimental Results on MNIST and EMNIST Datasets.}
Tables~\ref{tab:mnist_heterogeneity_d32} and~\ref{tab:emnist_heterogeneity_d32} report the \(RW_2\) angles for the full empirical distributions, individual digit classes, and selected class mixtures of MNIST and EMNIST, respectively.

For the MNIST dataset shown in Table~\ref{tab:mnist_heterogeneity_d32}, the \(W_2\)-nearest angle for the full distribution is larger than that of most individual digit classes, indicating stronger non-Gaussianity.
Across all evaluated distributions, the estimated \(W_2\)-nearest angle yields slightly smaller angles than the moment-matching baseline, providing a better Gaussian approximation at a higher computational cost.
Among all single-class distributions, digit~\(1\) has the smallest angle, \(\theta_{RW_2}=0.6675\), suggesting that it is closer to a single Gaussian-like source.

Furthermore, the selected class mixtures exhibit a consistent pattern related to the distributional heterogeneity: the \(RW_2\) angle increases from \(0.7694\) for digit \(\{0\}\), to \(0.7971\) for digits \(\{0,1\}\), \(0.8738\) for digits \(\{0,1,2,3\}\), and \(0.9203\) for digits \(\{0,1,2,3,4,5,6,7\}\). A similar trend is observed for another sequence of class mixtures, progressing from \(\{3\}\) to \(\{3,4\}\), then to \(\{3,4,5,6\}\), and finally to \(\{1,2,3,4,5,6,7,8\}\). These results suggest that combining more heterogeneous classes introduces additional distributional structure that is progressively less well represented by a single Gaussian. This indicates that the \(RW_2\) angle can potentially serve as a strong indicator of distributional heterogeneity.

For the EMNIST dataset shown in Table~\ref{tab:emnist_heterogeneity_d32}, the overall full dataset distribution has an \(RW_2\) angle value similar to that of the MNIST dataset; however, some individual data distributions differ.
For example, the \(RW_2\) angle of digit~\(4\) in the MNIST dataset is \(0.8294 \pm 0.0058\), which is smaller than that of digit~\(4\) in the EMNIST dataset, \(0.8523 \pm 0.0043\).
This indicates that digit~\(4\) in the MNIST dataset is closer to a single Gaussian-like source than digit~\(4\) in the EMNIST dataset, indicating that the former is less heterogeneous than the latter.

Overall, the \(RW_2\) angle provides an interpretable scalar summary of dataset non-Gaussianity and reveals differences across classes and selected mixtures, and shows a stronger indication of distributional heterogeneity for raw representations.

\paragraph{Experimental Results on CIFAR-10 and CINIC-10 Datasets.}
Tables~\ref{tab:cifar10_inception_heterogeneity_d32} and~\ref{tab:cinic10_inception_heterogeneity_d32} report the \(RW_2\) angles for the full empirical distributions, individual classes, and selected class mixtures of CIFAR-10 and CINIC-10, respectively.

For the CIFAR-10 dataset shown in Table~\ref{tab:cifar10_inception_heterogeneity_d32}, the \(W_2\)-nearest angle for the full distribution is larger than that of most individual classes, indicating stronger non-Gaussianity.
Across all evaluated distributions, the estimated \(W_2\)-nearest angle yields slightly smaller angles than the moment-matching baseline, providing a better Gaussian approximation at a higher computational cost.
Among all single-class distributions, the truck class has the smallest angle, \(\theta=0.7987\), suggesting that it is closer to a single Gaussian-like source.

Furthermore, the selected class mixtures exhibit a consistent pattern related to the distributional heterogeneity: the angle increases from \(0.8331\) for classes \(\{0,1\}\), to \(0.9015\) for \(\{0,1,2,3\}\), and further to \(0.9227\) for \(\{0,1,2,3,4,5,6,7\}\).
A similar trend is observed for CINIC-10, where the corresponding angles increase from \(0.8401\) for classes \(\{0,1\}\), to \(0.9023\) for \(\{0,1,2,3\}\), and then to \(0.9137\) for \(\{0,1,2,3,4,5,6,7\}\).
These results suggest that combining more heterogeneous classes introduces additional distributional structure that is progressively less well represented by a single Gaussian. These results suggest that the \(RW_2\) angle can serve as a useful indicator for capturing distributional heterogeneity.

For the CINIC-10 dataset shown in Table~\ref{tab:cinic10_inception_heterogeneity_d32}, the overall full dataset distribution has an \(RW_2\) angle value similar to that of the CIFAR-10 dataset; however, some individual data distributions differ.
For example, the \(W_2\)-nearest angle of the horse class in the CIFAR-10 dataset is \(0.8159 \pm 0.0089\), which is smaller than that of the horse class in the CINIC-10 dataset, \(0.8527 \pm 0.0078\).
This indicates that the horse class in the CIFAR-10 dataset is closer to a single Gaussian-like source than the horse class in the CINIC-10 dataset, indicating that the former is less heterogeneous than the latter.

Overall, the \(RW_2\) angle provides an interpretable scalar summary of dataset non-Gaussianity and reveals differences across classes and selected mixtures.
These results suggest that the \(RW_2\) angle can serve as a useful indicator for capturing distributional heterogeneity.

\begin{table}[ht]
    \centering
    \caption{Non-Gaussianity evaluation of CIFAR-10 dataset using pretrained Inception-v3 features reduced to \(d=32\) dimensions by PCA. We report the mean and standard deviation of the \(RW_2\) angles to the estimated \(W_2\)-nearest and moment-matching Gaussian approximations, together with the average running time for each method. Angles are measured in radians.}
    \label{tab:cifar10_inception_heterogeneity_d32}
    \resizebox{\textwidth}{!}{
    \begin{tabular}{llcccc}
        \toprule
        Distribution & Classes & MM angle & $W_2$-nearest angle & MM time / sec & $W_2$-nearest time / sec \\
        \midrule
        Full dataset & all & \(0.9673 \pm 0.0040\) & \(0.9228 \pm 0.0084\) & \( 0.0445 \) & \( 3.1869 \) \\
        Class airplane & 0 & \(0.9015 \pm 0.0029\) & \(0.8488 \pm 0.0040\) & \( 0.0343 \) & \( 3.1542 \) \\
        Class automobile & 1 & \(0.8825 \pm 0.0064\) & \(0.8266 \pm 0.0075\) & \( 0.0343 \) & \( 3.1504 \) \\
        Class bird & 2 & \(0.9385 \pm 0.0057\) & \(0.8898 \pm 0.0065\) & \( 0.0345 \) & \( 3.1625 \) \\
        Class cat & 3 & \(0.9352 \pm 0.0076\) & \(0.8874 \pm 0.0056\) & \( 0.0340 \) & \( 3.1564 \) \\
        Class deer & 4 & \(0.9226 \pm 0.0050\) & \(0.8681 \pm 0.0065\) & \( 0.0336 \) & \( 3.1553 \) \\
        Class dog & 5 & \(0.9013 \pm 0.0053\) & \(0.8451 \pm 0.0089\) & \( 0.0338 \) & \( 3.1396 \) \\
        Class frog & 6 & \(0.9148 \pm 0.0074\) & \(0.8590 \pm 0.0049\) & \( 0.0332 \) & \( 3.1471 \) \\
        Class horse & 7 & \(0.8719 \pm 0.0081\) & \(0.8159 \pm 0.0089\) & \( 0.0333 \) & \( 3.1318 \) \\
        Class ship & 8 & \(0.8763 \pm 0.0033\) & \(0.8216 \pm 0.0062\) & \( 0.0341 \) & \( 3.1255 \) \\
        Class truck & 9 & \(0.8601 \pm 0.0075\) & \(0.7987 \pm 0.0051\) & \( 0.0333 \) & \( 3.1118 \) \\
        Mixture & 0,1 & \(0.8864 \pm 0.0038\) & \(0.8331 \pm 0.0056\) & \( 0.0335 \) & \( 3.1322 \) \\
        Mixture & 0,1,2,3 & \(0.9491 \pm 0.0042\) & \(0.9015 \pm 0.0079\) & \( 0.0346 \) & \( 3.1767 \) \\
        Mixture & 0,1,2,3,4,5,6,7 & \(0.9716 \pm 0.0053\) & \(0.9227 \pm 0.0055\) & \( 0.0338 \) & \( 3.1654 \) \\
        Mixture & 3,4 & \(0.9289 \pm 0.0042\) & \(0.8773 \pm 0.0037\) & \( 0.0327 \) & \( 3.1188 \) \\
        Mixture & 3,4,5,6 & \(0.9354 \pm 0.0021\) & \(0.8927 \pm 0.0079\) & \( 0.0339 \) & \( 3.1211 \) \\
        Mixture & 1,2,3,4,5,6,7,8 & \(0.9674 \pm 0.0056\) & \(0.9183 \pm 0.0057\) & \( 0.0349 \) & \( 3.1764 \) \\
        Mixture & 5,6 & \(0.9156 \pm 0.0078\) & \(0.8608 \pm 0.0034\) & \( 0.0351 \) & \( 3.1533 \) \\
        Mixture & 5,6,7,8 & \(0.9329 \pm 0.0044\) & \(0.8861 \pm 0.0087\) & \( 0.0347 \) & \( 3.1736 \) \\
        Mixture & 2,3,4,5,6,7,8,9 & \(0.9607 \pm 0.0040\) & \(0.9161 \pm 0.0028\) & \( 0.0340 \) & \( 3.1728 \) \\
        \bottomrule
    \end{tabular}
    }
\end{table}

\begin{table}[ht]
    \centering
    \caption{Non-Gaussianity evaluation of CINIC-10 dataset using pretrained Inception-v3 features reduced to \(d=32\) dimensions by PCA. We report the mean and standard deviation of the \(RW_2\) angles to the estimated \(W_2\)-nearest and moment-matching Gaussian approximations, together with the average running time for each method. Angles are measured in radians.}
    \label{tab:cinic10_inception_heterogeneity_d32}
    \resizebox{\textwidth}{!}{
    \begin{tabular}{llcccc}
        \toprule
        Distribution & Classes & MM angle & $W_2$-nearest angle & MM time / sec & $W_2$-nearest time / sec \\
        \midrule
        Full dataset & all & \(0.9679 \pm 0.0050\) & \(0.9152 \pm 0.0025\) & \( 0.0401 \) & \( 3.0037 \) \\
        Class airplane & 0 & \(0.8886 \pm 0.0066\) & \(0.8319 \pm 0.0040\) & \( 0.0300 \) & \( 2.9551 \) \\
        Class automobile & 1 & \(0.8955 \pm 0.0038\) & \(0.8349 \pm 0.0036\) & \( 0.0304 \) & \( 2.9701 \) \\
        Class bird & 2 & \(0.9357 \pm 0.0043\) & \(0.8862 \pm 0.0040\) & \( 0.0299 \) & \( 3.0161 \) \\
        Class cat & 3 & \(0.9383 \pm 0.0049\) & \(0.8939 \pm 0.0008\) & \( 0.0298 \) & \( 2.9658 \) \\
        Class deer & 4 & \(0.9323 \pm 0.0079\) & \(0.8809 \pm 0.0091\) & \( 0.0288 \) & \( 2.9741 \) \\
        Class dog & 5 & \(0.9295 \pm 0.0030\) & \(0.8735 \pm 0.0046\) & \( 0.0294 \) & \( 2.9978 \) \\
        Class frog & 6 & \(0.9070 \pm 0.0034\) & \(0.8500 \pm 0.0051\) & \( 0.0299 \) & \( 2.9766 \) \\
        Class horse & 7 & \(0.9044 \pm 0.0064\) & \(0.8527 \pm 0.0078\) & \( 0.0293 \) & \( 2.9739 \) \\
        Class ship & 8 & \(0.9021 \pm 0.0042\) & \(0.8524 \pm 0.0042\) & \( 0.0299 \) & \( 2.9705 \) \\
        Class truck & 9 & \(0.8855 \pm 0.0067\) & \(0.8258 \pm 0.0058\) & \( 0.0299 \) & \( 2.9660 \) \\
        Mixture & 0,1 & \(0.8951 \pm 0.0030\) & \(0.8401 \pm 0.0075\) & \( 0.0306 \) & \( 2.9707 \) \\
        Mixture & 0,1,2,3 & \(0.9552 \pm 0.0087\) & \(0.9023 \pm 0.0065\) & \( 0.0299 \) & \( 2.9993 \) \\
        Mixture & 0,1,2,3,4,5,6,7 & \(0.9646 \pm 0.0064\) & \(0.9137 \pm 0.0013\) & \( 0.0310 \) & \( 3.0034 \) \\
        Mixture & 3,4 & \(0.9348 \pm 0.0029\) & \(0.8866 \pm 0.0027\) & \( 0.0303 \) & \( 2.9585 \) \\
        Mixture & 3,4,5,6 & \(0.9381 \pm 0.0050\) & \(0.8932 \pm 0.0052\) & \( 0.0303 \) & \( 2.9802 \) \\
        Mixture & 1,2,3,4,5,6,7,8 & \(0.9616 \pm 0.0056\) & \(0.9127 \pm 0.0060\) & \( 0.0303 \) & \( 2.9950 \) \\
        Mixture & 5,6 & \(0.9279 \pm 0.0053\) & \(0.8715 \pm 0.0073\) & \( 0.0299 \) & \( 2.9909 \) \\
        Mixture & 5,6,7,8 & \(0.9379 \pm 0.0063\) & \(0.8851 \pm 0.0090\) & \( 0.0302 \) & \( 2.9648 \) \\
        Mixture & 2,3,4,5,6,7,8,9 & \(0.9557 \pm 0.0034\) & \(0.9085 \pm 0.0024\) & \( 0.0294 \) & \( 2.9945 \) \\
        \bottomrule
    \end{tabular}
    }
\end{table}